\documentclass[accepted]{styles/uai2023} 

\usepackage[american]{babel}

\usepackage{natbib} 
\usepackage{mathtools} 
\usepackage{booktabs} 
\usepackage[normalem]{ulem}

\usepackage[capitalize,noabbrev]{cleveref}
\usepackage{amsmath}
\usepackage{amssymb}
\usepackage{mathtools}
\usepackage{amsthm}

\theoremstyle{plain}
\newtheorem{theorem}{Theorem}[section]

\newtheorem{lemma}[theorem]{Lemma}

\theoremstyle{definition}

\theoremstyle{remark}

\usepackage{amsfonts}
\usepackage{thmtools,thm-restate}
\usepackage{algorithm, algorithmic}
\usepackage{styles/jmei}

\usepackage{titletoc}

\title{KrADagrad: Kronecker Approximation-Domination Gradient\\Preconditioned Stochastic Optimization}

\author[1$\dagger$]{\href{mailto:j@mei.to}{Jonathan~Mei}}
\author[1$\dagger$]{Alexander~Moreno}
\author[1$\dagger$]{Luke~Walters}
\affil[1]{Independent Researcher}
\begin{document}
\maketitle
\footnotetext[2]{Funded by and conducted at Luminous Computing.}
\begin{abstract}
    Second order stochastic optimizers allow parameter update step size and direction to adapt to loss curvature, but have traditionally required too much memory and compute for deep learning. Recently, Shampoo \citep{gupta18shampoo} introduced a Kronecker factored preconditioner to reduce these requirements: it is used for large deep models \citep{anil20scalable} and in production \citep{anil22factory}.
    However, it takes inverse matrix roots of ill-conditioned matrices. This requires 64-bit precision, imposing strong hardware constraints. In this paper, we propose a novel factorization, Kronecker Approximation-Domination (KrAD). Using KrAD, we update a matrix that directly approximates the inverse empirical Fisher matrix (like full matrix AdaGrad), avoiding inversion and hence 64-bit precision. We then propose 
KrADagrad$^\star$, with similar computational costs to Shampoo and the same regret. Synthetic ill-conditioned experiments show improved performance over Shampoo for 32-bit precision, while for several real datasets we have comparable or better generalization.
\end{abstract}

\section{Introduction}
    \label{sec:intro}

Second order stochastic optimization methods adapt to loss curvature, allowing for smaller parameter update steps in regions where the gradient changes quickly, avoiding bouncing behavior, and larger ones in flat regions. Traditionally, they required storing and inverting the Hessian to update parameters: this requires quadratic memory and cubic computation in the number of parameters. Thus, methods using only a diagonal Hessian/Fisher approximation \citep{duchi11adaptive,kingma14adam} have dominated the field. However, diagonal preconditioners only scale gradients in the canonical basis, while full preconditioners can potentially perform scaling in a rotated basis aligning more closely with loss curavture.

Recently, Shampoo \citep{gupta18shampoo,anil20scalable} proposed approximating the (empirical) Fisher matrices using Kronecker factorized matrices. The matrix version of Shampoo factorizes the full preconditioner matrix into left and right Kronecker factors, which allows storing and inverting the smaller factors instead of the full matrix. For parameters $\W\in \mathbb{R}^{m\times n}$, this reduces computation costs from $O(m^3n^3)$ to $O(m^3+n^3)$ and storage costs from $O(m^2n^2)$ to $O(m^2+n^2)$. AdaGrad \citep{duchi11adaptive} uses regret bound techniques based on Online Mirror Descent (OMD) \cite{srebro11universality} designed for vector updates. To use these techniques for matrix/tensor updates, \citep{gupta18shampoo} exploit domination results that relate vector update preconditioners to their matrix/tensor counterparts. However, Shampoo still requires inverse matrix roots, which are numerically unstable or inaccurate for ill-conditioned matrices in 32-bit precision. For preconditioned gradient descent, it is important for Shampoo to maintain the accuracy of the smallest eigenvalues (largest when inverted). 
It thus needs 64-bit precision, which requires some combination of slow TPU-CPU data transfers, stale preconditioner matrices, or even new machine learning accelerator hardware \citep{anil20scalable} supporting fast 64-bit matrix multiplication or fast accurate eigendecomposition.

A primary motivator to use any optimizer is to reach the same quality solution in less time or reach a better solution that other optimizers fail to reach. Second order optimizers are currently not as popular as first order methods due to: a) inertia to adoption, with a lack of highly optimized implementations in all major ML frameworks; b) the added compute and memory requirements; c) numerical stability and consequently the additional considerations required to get them to work or to debug them (e.g. numerical linear algebra, computer number formats); and d) even though they sometimes reach a solution unreachable by 1st order methods (or the same solution with lower wall clock time (WCT) if properly optimized), they don’t consistently for every task or architecture.

The key tradeoff is (b, c) vs (d). \cite{shukla22understanding} of Weights and Biases noted that some of their customers using Shampoo do find solutions that generalize better than those found with ADAM for their real world tasks (d), but that 2nd order optimizers are still more expensive (b) and have additional considerations (c).

In this paper, we address these limitations (c) by introducing a novel factorization, Kronecker Approximation-Domination (KrAD): it has a simple form that updates the preconditioning matrix without explicitly inverting it. Shampoo constructs Kronecker factors of intermediate statistics such that their Kronecker product dominates the gradient outer product matrix. Our key idea is to construct factors for those statistics such that the \textit{inverse} of their Kronecker product dominates the gradient outer product matrix. This leads to preconditioners that require fractional powers rather than inverse fractional powers of factors. While this does not decrease the computational complexity compared to Shampoo, it avoids needing 64-bit precision.

This paper has three primary contributions: 1) we introduce two new algorithms, both with $O(m^3+n^3)$ and $O(m^2+n^2)$ computational and memory complexity, respectively and only requiring positive matrix roots, in contrast to previous work requiring inverse matrix roots; 2) we show domination properties and use them to prove that 
our algorithm, which has similar computation cost to Shampoo, 
achieves optimal regret; 3) we show empirically that in 32-bit precision, we outperform Shampoo in synthetic experiments and perform similarly on some real experiments.
We first describe some mathematical tools, set up the problem, and describe second order optimization and established results related to our method in Section~\ref{sec:background}. Then we present our method and its theoretical properties in Section~\ref{sec:kradagrad}. Next, we consider the practical implementation of our method in Section~\ref{sec:implementation}. Then, we show empirical results in Section~\ref{sec:experiments}. Finally, we discuss implications in Section~\ref{sec:conclusion}.

    \section{Background and related work}
    \label{sec:background}

Here, we set up the notation and the problem and then describe relevant related works. We briefly describe optimization for vector-valued parameters then extend the discussion to matrix-valued parameters. We can further generalize to tensors, but continue with matrices in the main text for clarity and leave the tensor formulation for Appendix~\ref{app:implementation_tensors}.

\subsection{Notation and preliminaries}
\label{subsec:notation}
    We use bold lower case letters to denote column vectors (e.g. $\g\in\Rbb^{n}$), bold upper case letters to denote matrices (e.g. $\G\in\Rbb^{m\times n}$), and calligraphic letters to denote matrices composed of stacking vectors of interest (e.g. $\Gcal_k=\begin{pmatrix}\g_k & \g_{k-1} &\ldots & \g_1 \end{pmatrix}\in \Rbb^{n\times k}$). For square matrix $\A\in\Rbb^{n\times n}$, let the trace be $\tr(\A)=\sum_{i=1}^{n}a_{i,i}$, where $a_{i,j}$ denotes the element in row $i$ and column $j$ of $\A$. For matrices $\A,\B\in\Rbb^{m\times n}$, let $\A\cdot\B=\tr(\A^\top\B)$ be the matrix (Frobenius) inner product, and the induced Frobenius norm $\|\A\|_F = (\A\cdot\A)^{1/2}$. We use $\|\A\|_2$ to denote the spectral norm, the largest singular value of $\A$. We write $\A\succeq 0$ to mean $\A$ is symmetric positive semi-definite (PSD), while $\A\succ 0$ means $\A$ is symmetric positive definite (PD). For two PSD matrices $\B\succeq \A$ means that $\B-\A \succeq 0$ (similarly for $\succ$). For a PSD matrix, take $\A=\V\boldsymbol{\Lambda}\V^\top$ to be the eigenvalue decomposition, which results in orthonormal $\V$ (i.e. $\V^{-1}=\V^\top$). Define $f(\boldsymbol{\Lambda})$ such that the diagonal elements are $(f(\boldsymbol{\Lambda}))_{i,i}=f(\lambda_{i})$, the principal values of the function applied to the scalar eigenvalues. Then we take $f(\A)=\V f(\boldsymbol{\Lambda})\V^\top$. In this way, we define a unique value for functions applied to PSD matrices with eigenvalues within the domain of the function. In particular, we have a definition for real powers of PD matrices.

    Let $\otimes$ denote the Kronecker product, for matrices $\A\in\Rbb^{m\times n}$ and $\B\in\Rbb^{q\times r}$ defined as\par
    {\centering
        $\B\otimes\A=
        \begin{pmatrix}
        b_{1,1}\A & \ldots & b_{1,r}\A \\
        \vdots & \ddots & \vdots \\
        b_{q,1}\A & \ldots & b_{q,r}\A
        \end{pmatrix} \in \Rbb^{mq\times nr}$.
    \par}
    Let the vectorization operation for a matrix $\A\in\Rbb^{m\times n}$ be
    \par
    {\centering
        $\vec(\A) = \begin{pmatrix} \a_{1}^\top & \ldots & \a_{n}^\top \end{pmatrix}^\top \in\Rbb^{mn}$:
    \par}
    $\a_i$ is the $i$-th column of $\A$, and the corresponding inverse vectorization for $\a\in\Rbb^{mn}$ given a target matrix in $\Rbb^{m\times n}$ is\par
    {\centering $\vec^{-1}_{m,n}(\a) = \begin{pmatrix} \a_{1:m} & \ldots & \a_{m(n-1)+1:mn} \end{pmatrix} \in\Rbb^{m\times n}$,
    \par}
    where $\a_{i:j}=\begin{pmatrix}a_i & \ldots & a_j\end{pmatrix}^\top$.
    
    Several properties \citep{bellman80some,vanloan00ubiquitous,boyd04convex,baumgartner11inequality} of trace and Kronecker products are important for our results. We list them in Appendix~\ref{app:derivations}
    due to space constraints.

\subsection{Optimization in machine learning}
\label{subsec:opt_in_ml}
    We are interested in iterative empirical risk minimization under loss $f$, with $\x\sim p_n(x)$ an empirical density and parameters $\w\in\Rbb^{N}$
    \par    
    {\centering
        $\w^* = \underset{\w}{\argmin}\, \mathbb{E}_{p_n}[f(\w, \x)]$.
    \par}
    We assume access to gradients $\nabla_\w f$. Let $\g_k = \sum_{i\in B_k}\nabla_\w f(\w_k,\x_i)\in \Rbb^{N}$ be the estimated gradient of $f$ w.r.t. $\w$ evaluated at $\w_k$ with data from batch $B_k$ at iteration $k$. From here, we omit but imply ``stochastic'' gradients.
    
    For step size $\eta_k\in \Rbb^+$, gradient-based methods update
    \begingroup\abovedisplayskip=3pt plus 1pt minus 3pt\belowdisplayskip=3pt plus 1pt minus 3pt
    \begin{align}
    \label{eq:preconditioned_update}
        \w_{k+1} = \w_k - \eta_k \P_k \g_k
    \end{align}
    \endgroup
    where $\P_k\in \Rbb^{N\times N}$ is a \textit{preconditioner} matrix (some sources instead refer to $\P_k^{-1}$ as the preconditioner). In some algorithms, additional intermediate \textit{statistics} are stored and updated to aid preconditioner computation. Gradient descent uses $\P_k = \I_N$ (no preconditioning); Newton's method takes $\P_k=\mathbf{H}_k^{-1}$ to be the (pseudo)inverse of the Hessian.

    While vanilla gradient descent updates are trivial to compute, convergence can require many iterations. Newton updates are more expensive, but may require far fewer iterations. In practice, the chosen form of the preconditioner matrix appears to exist along a trade-off between computational tractability and improved convergence properties.

\subsection{Adaptive gradient preconditioners}
\label{subsec:ada_grad_preconditioners}
    We can collect gradients through iteration $k$,\par
    {\centering    $\Gcal'_k = \begin{pmatrix}\g_k & \g_{k-1} &\ldots & \g_i &\ldots & \g_1 \end{pmatrix}\in \Rbb^{N\times k}$ \par}
    and augment this with a scaled identity,\par
    {\centering
        $\Gcal_k = \begin{pmatrix}
            \Gcal'_k & \epsilon \I_N
        \end{pmatrix}\in \Rbb^{N\times (k+N)}$.
    \par}
    One form of adaptive gradient update is\par
    {\centering
        ${\small \w_{k+1} = \w_k - \eta_k \left(\Gcal_{k} \Gcal_{k}^\top\right)^{-1/2} \g_k}$.
    \par}
    Expressing this in terms of the non-augmented $\mathcal{G}'_k$ and taking $\delta=\epsilon^2$ gives the full version of AdaGrad~\citep{duchi11adaptive}: $\Gcal_{k} \Gcal_{k}^\top$ can be seen as the statistic that is stored and updated in each iteration, and $\big(\delta\I_N +\Gcal'_{k} \Gcal_{k}^{'\top}\big)^{-1/2}$ is the preconditioner computed from the statistic. Unfortunately, storing the full matrix $\Gcal_{k} \Gcal_{k}^\top$ is memory intensive at $O(N^2)$, and taking the inverse square root is computationally expensive at $O(N^3)$ to compute the SVD.

    Diagonal AdaGrad reduces computational complexity\par
    {\centering
        $\w_{k+1} = \w_k - \eta_k \left(\delta\I_N + \textrm{diag}\left(\Gcal'_{k} \Gcal_{k}^{'\top}\right)\right)^{-1/2} \g_k$.
    \par}
    This is $O(N)$ complexity in both memory and computation.

\subsection{Matrix variables and Shampoo}
    \label{subsec:mat_vars}
    Now, we consider $N=mn$ and optimize w.r.t. a matrix $\W\in\Rbb^{m\times n}$. We consider a single matrix for clarity, but note that the derivations and analyses can be extended to tensors and applied individually to each tensor-valued parameter in a given model (e.g. to compute the total costs). Then we can use the same optimization framework, now taking $\w=\vec(\W)\in\Rbb^{mn}$. However, in forming a preconditioner, we utilize the fact that our parameter now has the additional structure of being a matrix. Consider\par
    {\centering $\G_k = \vec^{-1}_{m,n}(\g_k)$.
    \par}
    
    One convenient factorized form of a preconditioner is\par
    {\centering 
        $\P_k = \R_k \otimes \L_k$,
    \par}
    where $\L_k\in\Rbb^{m\times m}$ and $\R_k\in\Rbb^{n\times n}$ are symmetric. This reduces storage and computation while not necessarily being low-rank. To see this, we return to Equation~\eqref{eq:preconditioned_update} and simplify
    \begin{align*}
        \w_{k+1} &= \w_k - \eta_k \P_k \g_k \nonumber\\
        &=\w_k - \eta_k \vec(\L_k\G_k\R_k) \qquad \because (\textrm{P}18) 
        \nonumber\\
        \Rightarrow \W_{k+1} & = \W_k - \eta_k \L_k\G_k\R_k. \qquad \because (\vec^{-1}_{m,n})
    \end{align*}
    
    This requires $O(m^2 + n^2)=O\left(N\left(\frac{m}{n}+\frac{n}{m}\right)\right)$ storage and $O(N(m+n))$ compute.
    Unless otherwise specified, we assume w.l.o.g. that $m\le n$. Shampoo~\citep{gupta18shampoo} tracks statistics
    \begingroup\abovedisplayskip=3pt plus 1pt minus 3pt\belowdisplayskip=3pt plus 1pt minus 3pt
    \begin{align}
        {\small \B_{k} =\epsilon\I_m + \textrm{$\sum_{i=1}^{k}$} \G_i \G_i^\top}, \label{eq:shampoo_stat1} \quad
        {\small \C_{k} =\epsilon\I_n + \textrm{$\sum_{i=1}^{k}$} \G_i^\top \G_i} ,
    \end{align}
    \endgroup
    and forms the preconditioner from the Kronecker factors
    \begingroup\abovedisplayskip=3pt plus 1pt minus 3pt\belowdisplayskip=3pt plus 1pt minus 3pt
    \begin{align}
    \label{eq:shampoo_precon1}
        (\L_{k}, \R_{k}) &= (\B_k^{-1/4}, \C_k^{-1/4}).
    \end{align}
    \endgroup
    
    \cite{gupta18shampoo} relies on 3 key conditions to prove Shampoo achieves optimal regret. First,
    \begingroup\abovedisplayskip=3pt plus 1pt minus 3pt\belowdisplayskip=3pt plus 1pt minus 3pt
    \begin{align}
    \label{eq:shampoo_dominate_grad}
        \R_k^{-2} \otimes \L_k^{-2} &\succeq \epsilon\I_{mn} + \sum_{i=1}^{k} \g_i \g_i^\top.
    \end{align}
    \endgroup
    It secondly requires that
    \begingroup\abovedisplayskip=3pt plus 1pt minus 3pt\belowdisplayskip=3pt plus 1pt minus 3pt
    \begin{align}
    \label{eq:shampoo_dominate_prev}
        \R_k^{-2} \otimes \L_k^{-2} &\succeq \R_{k-1}^{-2} \otimes \L_{k-1}^{-2}.
    \end{align}
    \endgroup
    %
    %
    %
    %
    Finally it requires that under mild conditions,
    \begin{align}
    \label{eq:shampoo_trace_rate}
        {\small \tr(\L_k), \tr(\R_k) = O(k^{1/4})}.
    \end{align}
    An additional $O(n^3+m^3)$ cost comes from taking the inverse fractional powers via a high-precision (64-bit) Newton iteration (which involves repeated matrix multiplications; see Equation \eqref{eq:coupled_newton_iteration} in Appendix \ref{app:implementation_matrix_roots} for further details) or SVD, which dominates the previous $O(N(m+n))$.

    The motivation for these properties stems from OMD analysis. The regret vector parameter $\w_k$ updates and general preconditioners $\P_k \succ 0$ is initially bounded by sums of quadratic forms $\w_k^\top \P_k \w_k$. If the domination property holds, these can be bounded in terms of $\tr(\P_k)$. Using Property~(P13), 
    this can be further expressed in terms of $\tr(\L_k)$ and $\tr(\R_k)$. The trace growth rates give the final bound.

    \subsection{Related Work}
    \label{subsec:related_work}

    Recently, there has been a surge in interest in tractable preconditioned gradient methods. We briefly contrast some of the most similar or otherwise notable methods to ours.
    
    \textbf{Kronecker Factored}: KFAC \citep{martens15optimizing} and TNT \citep{ren21tensor} use Kronecker factors to approximate Fisher matrices while reducing storage and computation costs. KFAC requires knowledge of network architecture and thus modifications or even re-implementations corresponding to each parametric layer type within the network. KFAC and TNT reqiore an additional backward pass and matrix inversion. KBFGS \citep{goldfarb20practical} does not require matrix inversion but requires an additional forward and backward pass. We note that the empirical performance achieved by KBFGS is partially due to initialization using curvature estimated from the entire training set \citep{goldfarb20practical}, which is not available in a truly online setting. Shampoo is the most closely related work, relying on the empirical Fisher matrix rather than estimating the Fisher matrix, thus not requiring additional sampling or forward/backward passes. In addition, only tensor shape knowledge is required. While estimating the Fisher matrix may have intuitively desirable properties over the empirical Fisher, the empirical Fisher is more practical to compute \citep{martens20new}, since in distributed data or model parallel settings, additional forward/backward passes become prohibitive (both in terms of computation and engineering cost). To our knowledge, Shampoo is the only second order optimizer that has been successfully implemented in a large-scale, production, deep learning setting \citep{anil22factory}, which makes it of primary interest.
    
    \textbf{Limited Memory}: GGT~\citep{agarwal19efficient} uses a limited history of $h$ past gradients to form a low-rank approximation to the full AdaGrad matrix, reducing storage costs to $O(Nh)$ and compute costs to $O(Nh^2)$; however, this requires many copies ($200$, for the problems they consider) of the full gradient to be stored as statistics ($h$ still scales as a function of $N$), which can become prohibitive without modifications to reduce $N$.
    
    \textbf{Sketching}: AdaHessian \citep{yao21adahessian}, SENG \citep{yang22sketch}, and SketchySGD \citep{frangella22sketchysgd} estimate the Hessian (either the diagonal or a low-rank approximation) via automatic differentiation to compute Hessian-vector products (HVP), which require additional backpropagation steps and two batches of data per step, one for gradient computation and one for Hessian sketching. This is less expensive than a full forward/backward pass, but can still be expensive in distributed settings. In addition, the low-rank factorization still requires many times the storage of the full gradient of the model ($100-200\times$ in the problems they consider). While KrADagrad does not use sketching or HVP, these methods could potentially be combined with KrAD factorization as future work.
    
    \section{KrADagrad}
    \label{sec:kradagrad}
    Here, we derive a pair of new optimization algorithms, KrADagrad and KrADagrad$^\star$, presenting along the way intermediate results that allow us to attain domination properties analogous to those in Equations~\eqref{eq:shampoo_dominate_grad}-\eqref{eq:shampoo_dominate_prev} and trace growth rates in Equation~\eqref{eq:shampoo_trace_rate} required for good regret. Simultaneously, this maintains low computational complexity achieved by Kronecker factorized methods. To derive the algorithms, we present 
\begin{enumerate}
    \item KrAD, a method for producing a Kronecker factorization approximating a matrix that yields the property in Equation \eqref{eq:shampoo_dominate_grad}
    \item Derivation of the basic form of KrADagrad, applying to AdaGrad style updates KrAD and the Woodbury matrix identity \cite{woodbury50inverting} as the key tricks
    \item Statements confirming the property in Equation \eqref{eq:shampoo_dominate_prev}
    \item Extension to KrADagrad$^\star$.
\end{enumerate}

KrADagrad alternates updates of Kronecker factors of the statistics and has within $\varepsilon$ tolerance of optimal regret. With additional insights from KrADagrad regret analysis, we formulate KrADagrad$^\star$, which can be seen as an ``average'' of two KrADagrad estimators. KrADagrad$^\star$ updates both Kronecker factors of the statistics simultaneously and achieves optimal regret. 
We present theoretical results along the way in this section as they are needed but defer the proofs to Appendix~\ref{app:proofs}.
We derive the algorithm for matrix-valued parameters for clarity and again leave the extension to tensor-valued parameters for Appendix~\ref{app:implementation_tensors}.

\subsection{Kronecker Approximation-Domination}
\label{subsec:kron_approx_dom}
    First, we state a Lemma as the goal of KrAD, which is needed to achieve the condition in Equation \eqref{eq:shampoo_dominate_grad} that allows us to prove optimal regret.
    \begin{restatable}{lemma}{krad}
    \label{lem:krad}
        Let PD matrix $\C\in\Rbb^{n\times n}$, $\Ucal\in\Rbb^{mn\times r}$, $\u_i$ be the $i$-th column of $\Ucal$,  $\U_i=\vec^{-1}_{m,n}(\u_i)$. Then
        \begingroup\abovedisplayskip=2pt plus 1pt minus 2pt\belowdisplayskip=2pt plus 1pt minus 2pt
        \begin{align*}
            {\smaller\B} &{\smaller = \sum\limits_{i=1}^{r} \U_i\C^{-1}\U_i^\top \succeq 0} \\
        {\smaller\Rightarrow \C\otimes\B} &{\smaller \succeq \frac{1}{n}\Ucal\Ucal^\top.}
        \end{align*}
        \endgroup
    \end{restatable}
    These matrices $\U$ and $\Ucal$ are fairly general. In our setting, we will use KrAD on gradient matrices. In context, this result states that given a PD right matrix, we can express a left matrix in a quadratic form, and the Kronecker product of the right and left matrices will dominate the scaled gradient outer product matrix. We present the proof in Appendix~\ref{app:proof_krad}.

\subsection{KrADagrad updates: derivation}
\label{subsec:KrADagrad_updates}
    We start with deriving the general statistic update. We outline it here and fill in details in Appendix~\ref{app:derivations}.
    Suppose we have the previous Kronecker factorization of a statistic $\Q_{k-1}$ that dominates the gradient outer product matrix by a factor $t$ (which we will clarify later) and its inverse $\P_{k-1}$
    \begingroup\abovedisplayskip=3pt plus 1pt minus 3pt\belowdisplayskip=3pt plus 1pt minus 3pt
    \begin{align}
        \Q_{k-1} &= \C_{k-1}\otimes\B_{k-1} \succeq \frac{1}{nt}\Gcal_{k-1} \Gcal_{k-1}^\top \\
        \P_{k-1} &= \R_{k-1}\otimes\L_{k-1} = \C_{k-1}^{-1}\otimes\B_{k-1}^{-1} = \Q_{k-1}^{-1},
    \end{align}
    \endgroup
    where $\B_{k-1}$, $\C_{k-1}$, $\L_{k-1}$, and $\R_{k-1}$ are PD. 
    Our initial intermediate update, which we will apply our KrAD factorization to, is\par 
    {\centering $\widetilde\Q_{k} = \frac{1}{nt_k}\g_{k}\g_{k}^\top + \Q_{k-1},$ \par}
    for some $t_{k}\le t$. 
    We will then compute an intermediate version of the update on the inverse of the left statistic $\widetilde\B_{k}=\widetilde\L_{k}^{-1}$. Letting $\widetilde\B_k=\B_{k-1} + \Delta\B_k$, we can apply KrAD to $\widetilde\Q_{k}$ with a fixed $\C_k=\C_{k-1}$ (i.e., we use the old right statistic to update the current left one),
    \begingroup\abovedisplayskip=3pt plus 1pt minus 3pt\belowdisplayskip=3pt plus 1pt minus 3pt
    \begin{align}
        &\Delta\B_{k} = \frac{1}{t_k}\G_k\C_{k}^{-1}\G_k^{\top}\\
        \Rightarrow&\C_k \otimes \Delta\B_{k} \succeq \frac{1}{nt_k}\g_k\g_k^{\top} \qquad \because \textrm{Lemma \ref{lem:krad}}\\
        \Rightarrow& \C_k \otimes (\B_{k-1} + \Delta\B_{k}) \succeq \frac{1}{nt}\Gcal_k\Gcal_k^\top \quad \because (\textrm{P15})
        \label{eqn:b-tilde-dominating-scaled-gradients}.
    \end{align}
    \endgroup
    Then, letting $\widehat{\C}_k = \C_{k}\!+\!\frac{1}{t_k}\G_k^{\top}\L_{k\!-\!1}\G_k$ and applying the Woodbury matrix identity,
    \begingroup\abovedisplayskip=3pt plus 1pt minus 3pt\belowdisplayskip=3pt plus 1pt minus 3pt
    \begin{align}
        \widetilde\L_{k}&=\widetilde\B_{k}^{-1}=(\B_{k-1}+\Delta\B_{k})^{-1} \\
        &\succeq \underbrace{\L_{k\!-\!1} \!-\! \frac{1}{t_k}\L_{k\!-\!1}\G_k\R_{k}\G_k^{\top}\L_{k\!-\!1}}_{\L_k} \label{eq:krad_update_dominate_inverse}
    \end{align}
    \endgroup
    (we provide more detail in Appendix~\ref{app:derivations}).
    Note that our update for $\L_k$ neither depends on $\B$, $\C$ nor requires any other expensive matrix inverses to compute. This suggests that we do not need to actually store $\B$ or $\C$ to obtain a computationally tractable implementation.

    Here, we state intermediate results that suggest that our proposed updates are reasonable (proof in Appendix~\ref{app:proof_middle_term}).
    \begin{restatable}{proposition}{propmiddleterm}
    \label{prop:middle_term}
        Taking $t_k=1+\|\L_{k-1}\G_k\R_{k-1}\G_k^\top\|_2$ (or the looser but more computationally friendly $t_k=1+\|\L_{k-1}\G_k\R_{k-1}\G_k^\top\|_F$), the PSD matrix\par
        {\centering $\M_k \overset{\Delta}{=} \frac{1}{t_k}\L_{k-1}^{1/2}\G_{k}\R_{k-1}\G_{k}^\top \L_{k-1}^{1/2} \prec \I.$ \par}
    \end{restatable}
    \begin{restatable}{corollary}{corKrADupdatePD}
    \label{cor:KrAD_update_PD}
        If $\L_{k-1}\succ 0$, the updated $0 \prec \L_k \preceq \L_{k-1}$.
    \end{restatable}
    In Corollary~\ref{cor:KrAD_update_PD}, the second inequality $\L_k \preceq \L_{k-1}$ has the effect of not increasing the step size, while the first inequality $\L_k \succ 0$ guarantees that we do not reverse the direction of the gradient. These are both valued theoretical properties of useful preconditioners for avoiding divergent behavior. In practice, they may be lightly violated to great effect; for example, in Adam, it is technically possible to have increasing step size as shown by~\citep{reddi18convergence}. The result from Corollary~\ref{cor:KrAD_update_PD} also allows us to leverage existing techniques to bound the regret of our algorithm.

    \subsubsection{Update schemes}
    Thus far, we have glossed over the fact that we have actually only updated $\L_k$. We may update $\R_k$ in the same way; since we have already achieved domination by just updating $\L_k$, we are interested in the process for jointly updating $(\L_k, \R_k)$. Further, we have only hinted at a method for updating statistics; we must still compute the preconditioner.
    
    In the next few subsections, we discuss these issues more concretely, proposing KrADagrad$^\star$, an algorithm that combines two sets of KrAD preconditioners to obtain optimal regret but requires updating two matrix statistics and involves higher order matrix roots. This algorithm is reminiscent of Shampoo, but avoids the numerical difficulty of inverting ill-conditioned matrices. We also propose, KrADagrad, a separate scheme that has suboptimal regret but is more intuitive, showing why we arrive at this form of update. Due to space constraints, we leave this to Appendix~\ref{app:implementation_kradagrad}.

\subsection{KrADagrad$^\star$: combining preconditioners}
    Suppose we have two distinct sets of KrADagrad preconditioners. Here we overload notation a bit, holding the iteration $k$ fixed and dropping it from the subscript, instead using the subscript to denote the index for the set of preconditioners to which each matrix belongs,\par
    {\centering        $\R_1 \otimes \L_1 \succeq \Gcal\Gcal^\top, \qquad        \R_2 \otimes \L_2 \succeq \Gcal\Gcal^\top$,
    \par}
    recalling that $\Gcal=\begin{pmatrix}
        \Gcal^{'} & \epsilon \I_{N}
    \end{pmatrix}$, the augmented matrix collecting the history of observed gradients.
    Using the matrix geometric mean for two matrices \citep{ando04geometric},\par
    {\centering
        $\begin{aligned}M_g(\A, \B) &\overset{\Delta}{=} \A^{1/2}(\A^{-1/2}\B\A^{-1/2})^{1/2}\A^{1/2}\\
        & = \A (\A^{-1}\B)^{1/2}\end{aligned}$
    \par}
    and due to the geometry of the manifold of PD matrices \citep{bhatia09positive},\par
    {\centering
        $\R_c \otimes \L_c \succeq \Gcal\Gcal^\top$.
    \par}
    where $\L_c=M_g(\L_1, \L_2)$ and $\R_c=M_g(\R_1, \R_2)$ form another pair of KrAD estimates\footnote{the subscript $c$ stands for ``combined''}.
    
    In this case, we may need to take additional square roots of the quantities $(\L_1^{-1}\L_2, \R_1^{-1}\R_2)$ and of $(\L_c,\R_c)$ themselves. We will first show how these combined estimators are relevant to Shampoo, and use this insight to arrive at a second version of preconditioning, which we call KrADagrad$^\star$.
    
    \subsubsection{Shampoo combines inverse KrAD estimates}
        If we keep a pair of KrAD estimates for $\Q$ (an integral power of the preconditioner inverse) instead of directly for $\P$, one in which we only update $\B_1$ and keep $\C_1=\I$ fixed, while in the other we only update $\C_2$ and keep $\B_2=\I$ fixed, we end up with exactly the Shampoo statistics updates\par
        {\centering
            $\begin{aligned}[b]\Delta\B_1 &= \G\C_1^{-1}\G^{\top}=\G\G^{\top} \\
            \Delta\C_2 &= \G^{\top}\B_2^{-1}\G=\G^{\top}\G\end{aligned}$.
        \par}
        Then, the full Shampoo statistics matrices are related to a combination of these two statistics,\par
        {\centering
            $\begin{aligned}[b]\B_c &= \! M_g(\B_1,\I) \!=\! \B_1^{1/2}\\
            \C_c &= \! M_g(\I,\C_2) \!=\! \C_2^{1/2}\end{aligned}$,
        \par}
        These still need to be inverse square rooted to be applied as the preconditioner, since here we have constructed\par
        {\centering
            $\C_c\otimes\B_c \succeq \Gcal\Gcal^{\top}$.
        \par}
        The inverse square root ultimately gets grouped together with the square roots in the expressions above to yield the inverse $1/4$-th power in the preconditioning,\par
        {\centering
            $\begin{aligned}\W_k &= \W_{k-1} - \eta\B_c^{-1/2}\G_k\C_c^{-1/2}\\
            &= \W_{k-1} - \eta\B_1^{-1/4}\G_k\C_2^{-1/4}\end{aligned}$
        \par}
        (compare with Equations~\eqref{eq:shampoo_stat1}-\eqref{eq:shampoo_precon1}). Each individual estimator keeps one Kronecker factor as identity.
        
    \subsubsection{KrADagrad$^\star$}
        Inspired by Shampoo's optimality, we can maintain a pair of KrADagrad preconditioners $(\L_{k,1},\R_{k,1})=(\L_{k},\I_m)$ and $(\L_{k,2},\R_{k,2})=(\I_n,\R_{k})$, where the second index in the LHS subscripts denotes the estimator. Since we will hold the identity matrices constant, we do not need to store or perform multiplication with them explicitly. In addition, this means we can unambiguously drop the second index in the subscripts. On iteration $k$, we update both $\L_{k}$ and $\R_{k}$, just as in shampoo. With\par
        {\centering
            $\begin{aligned}[b]t_{k,L} & \leftarrow 1+\|\G_k\G_k^\top \L_{k-1}\|_F \\
            t_{k,R} & \leftarrow 1+\|\G_k^\top \G_k\R_{k-1}\|_F\end{aligned}$,
        \par}
        we summarize in Algorithm \ref{alg:KrADagrad_star}.
        \begin{algorithm}[tb]
           \caption{KrADagrad$^\star$: Precondition \textit{without} inversion}
           \label{alg:KrADagrad_star}
            \begin{algorithmic}
                \STATE {\bfseries Input:} Parameters $\W_0\in\Rbb^{m\times n}$, iterations $K$, step size $\eta$, exponent $\alpha = 1/2$.
                \STATE {\bfseries Initialize:} $(\L_0, \R_0)=(\I_m, \I_n)$
                \FOR{$k=1, \ldots, K$}
                    \STATE Obtain gradient $\G_k$
                    \STATE Compute $\Delta\L_{k}=\frac{1}{t_{k,L}}\L_{k-1}\G_{k}\G_{k}^\top \L_{k-1}$
                    \STATE Compute $\Delta\R_{k}=\frac{1}{t_{k,R}}\R_{k-1}\G^\top_{k}\G_{k}\R_{k-1}$
                    \STATE Update $(\L_k,\R_k) \leftarrow (\R_{k-1}-\Delta\R_{k}, \L_{k-1}-\Delta\L_{k})$
                    \STATE Compute $(\L^{\alpha/2}_k, \R^{\alpha/2}_k)$ from $(\L_k, \L^{\alpha/2}_{k\!-\!1}, \R_k, \R^{\alpha/2}_{k\!-\!1})$
                    \STATE Apply preconditioned gradient step\par
                        {\centering
                            $\W_k = \W_{k-1} - \eta\L^{\alpha/2}_k\G_k\R^{\alpha/2}_k$
                        \par}
                \ENDFOR
            \end{algorithmic}
        \end{algorithm}
        In order to bound the regret, we need a few intermediate results.
        \begin{restatable}{lemma}{lemtracebound}
        \label{lem:trace_bound}
            Assume a 1-Lipschitz loss, implying $\|\G_k\|_2 \le 1 \forall k.$
            Suppose by iteration $k$, $\L_k$ is updated in $k - k'$ of those steps for $0\le k' \le k$ (and thus $\R_k$ is updated in the remaining $k'$ steps). Letting $\B_k=\L_k^{-1}$, $\C_k=\R_k^{-1}$, and $s=1 \!+\! \|\L_0\|_2\|\R_0\|_2$,
            \begingroup\abovedisplayskip=3pt plus 1pt minus 3pt\belowdisplayskip=3pt plus 1pt minus 3pt
            \begin{align}
                \tr(\B_k) & \le\! \tr(\B_0) \!+\! k' ms\|\R_0\|_2 \label{eq:trace1}\\
                \tr(\C_k) & \le\! \tr(\C_0) \!+\! (k \!-\! k') ns\|\L_0\|_2. \label{eq:trace2}
            \end{align}
            \endgroup
        \end{restatable}
        Next, we restate Theorem 7 from~\cite{gupta18shampoo} in our notational setting. We additionally make one minor but straightforward substitution of the smaller matrix dimension instead of the rank, as the dimensions upper bound the rank. 
        \begin{lemma}[From~\cite{gupta18shampoo}]
        \label{lem:gupta}
            Let $\w_{*}\in\Rbb^{mn}$ with $m\ge n$, $t>0$, and $D\overset{\Delta}{=}\max_{1\le k \le K}\|\w_k-\w_{*}\|_2$. The regret from using a Kronecker factorized preconditioner\par
            {\centering {\small $\P_k=\C_{k}^{-1/2}\otimes\B_{k}^{-1/2}$} \par}
            that dominates the empirical Fisher matrix\par
            {\centering {\small $\C_{k} \otimes \B_{k} \succeq \frac{1}{nt} \Gcal_k \Gcal_k^\top$} \par}
            is bounded\par
            {\centering {\small $\sum_{k=1}^{K}(f_k(\w_k)-f_k(\w_{*})) \le D\sqrt{2tn}\tr(\B_K^{1/2})\tr(\C_K^{1/2})$ } \par}
        \end{lemma}
        
        Now, we have our regret result: a proof is in Appendix~\ref{app:proof_krad_regret}.
        \begin{theorem}
        \label{thm:krad_regret}
            Assuming a 1-Lipschitz loss, the regret from using KrADagrad$^\star$ scales as $O(\sqrt{K})$.
        \end{theorem}

    \section{Implementation}
    \label{sec:implementation}
Now we discuss the algorithmic considerations for actually implementing KrADagrad$^\star$. The main difficulty lies in efficiently computing the matrix roots. Differentiable matrix square roots for machine learning applications have been the subject of a substantial amount of research~\citep{song21approximate}. For preconditioning, we do not require differentiability for our roots, so they can be computed using numerical linear algebra methods (e.g. SVD) without concern for the backward pass of the root itself. While such algorithms are gaining hardware support on current GeMM accelerators and software frameworks, they are still not universal (Pytorch supports SVD and QR on CUDA in double precision via cuSOLVER and MAGMA \citep{pytorch21blog,cusolver23api,dongarra14accelerating}, but TPUs to our knowledge do not \citep{jouppi21ten}).

However, as~\citep{anil20scalable} discovered, accurate inverse powers require high precision to retain the important contributions of the eigenvectors corresponding to the smallest eigenvalues, and current general matrix multiplication (GeMM) accelerators do not prioritize 64-bit computation performance (i.e. GPUs see significant speed reductions, while TPUs do not support it at all). Plus, they assume the preconditioners are a slowly-varying sequence of matrices. Thus, they propose computing matrix roots iteratively on CPU due to these algorithmic and hardware architectural constraints.


In contrast, Kradagrad$^\star$ deals with positive roots, and both these iterative matrix methods and the numerical linear algebra techniques are otherwise actually amenable to computation using GeMM accelerators. Ultimately, as a straightforward solution for the roots that appear in KrADagrad$^\star$, the SVD on GPU suffices. We note that avoiding inversion is thus not about reducing the computational complexity, rather we aim to avoid needing 64-bit computation. Further details of matrix roots are discussed in Appendix~\ref{app:implementation_matrix_roots}. 

Additionally, diagonal damping is a common feature of preconditioned methods due to their contribution to numerical stability. While we do not need this for numerical stability, we find empirically that diagonal damping still helps reduce the effect of gradient noise and smooth out the loss curve. We describe how this can be applied to KrADagrad$^\star$ in further detail in Appendix~\ref{app:implementation_damping}.

So far, our proposed algorithms and analyses apply to matrix parameters. As noted earlier, we can extend this to tensor parameters in a fairly straightforward manner as mainly a matter of additional notation and bookkeeping. Hence, we relegate the extension to Appendix~\ref{app:implementation_tensors}.

The baseline Pytorch Shampoo implementation we use is not optimized, so is not totally fair to the true capability of the baseline. Similarly, our KrADagrad variants, being derived from the mentioned implementation of Shampoo, did not have optimized implementations.
Nonetheless, we have conducted some preliminary wall clock time comparisons between our implementation of KrADagrad and Shampoo, which we summarize in Table~\ref{tab:wct}. For each data set, we report the time in seconds to run a single epoch for (KrADagrad, KrADagrad$^\star$, Shampoo) along with a 2 standard deviation interval, computed from epochs 5 through 10 on an NVidia A40 GPU.

\setlength{\belowcaptionskip}{-3pt}
\begin{table}[ht]
    \caption{Comparison of wall clock times (in seconds) between our implementations of KrADagrad variants and Shampoo on a single epoch and a 2 standard deviation interval computed from 5 epochs for CIFAR-10/100 data sets on an NVidia A40 GPU.}
    \label{tab:wct}
    \vspace{-10pt}
    \begin{center}
    \begin{smaller}
    \begin{sc}
    \begin{tabular}{lcccr}
        \toprule
        Data set & KrADagrad & KrADagrad$^\star$ & Shampoo \\
        \midrule
        CIFAR-10    & 29.50$\pm$ 0.88& 45.76$\pm$ 0.20& 35.62 $\pm$ 0.17 \\
        CIFAR-100 & 54.71$\pm$ 0.36& 76.12$\pm$ 1.72& 66.45 $\pm$ 0.33 \\
        \bottomrule
    \end{tabular}
    \end{sc}
    \end{smaller}
    \end{center}
    \vspace{-16pt}
\end{table}

\subsection{Compute and memory costs}
    For Kradagrad, computing $\Delta\L_k$ has a computational cost of $(2N(2m+n)+2m^2+2m^3)$
    While we still require $O(m^3+n^3)$ matrix square roots for the KrADagrad update, this is less difficult numerically and thus computationally than the $-4$th root required by Shampoo.
    
    KrADagrad$^\star$ requires 4th root computation, which is comparable in cost to the inverse 4th root, but without the need for high precision.

    Storage involves tracking the two matrix factors, and so is $O(m^2+n^2)$ for all methods.
    
    We summarize these costs in Table~\ref{tab:complexity}.
    \begin{table}[ht]
        \caption{Comparison of complexity between our implementations of KrADagrad variants and Shampoo.}
        \label{tab:complexity}
        \vspace{-10pt}
        \begin{center}
        \begin{smaller}
        \begin{sc}
        \begin{tabular}{lcccr}
            \toprule
            &KrADagrad & KrADagrad$^\star$ & Shampoo \\
            \midrule
            Compute    & $O(m^3 + n^3)$ & $O(m^3 + n^3)$ & $O(m^3 + n^3)$ \\
            Memory &  $O(m^2 + n^2)$ & $O(m^2 + n^2)$ & $O(m^2 + n^2)$ \\
            \bottomrule
        \end{tabular}
        \end{sc}
        \end{smaller}
        \end{center}
        \vspace{-16pt}
    \end{table}

    \section{Experiments}
    \label{sec:experiments}
    We address the following: 1) How sensitive is our method to matrix conditioning compared to Shampoo? 2) How does convergence speed compare in the number of training steps? (as measured by task-specific validation metrics) 3) How do our methods compare in the achieved model quality at or near convergence? The goal of each is to compare optimizers in various challenging loss landscapes, rather than to achieve state of the art performance.

To answer 1, in a synthetic experiment we minimize a multidimensional quadratic function with a non-diagonal, poorly-conditioned Hessian that neatly factorizes into a Kronecker product of two individually poorly-conditioned PD matrices.

To answer 2 and 3, we compare KrADagrad and KrADagrad$^\star$ to alternatives across a variety of tasks: image classification (IC), autoencoder problems (AE), recommendation (RecSys), continual learning (CL). For IC experiments we train ResNet-32/56~\citep{he16deep} without BatchNorm (BN) on CIFAR-10/100~\citep{krizhevsky09learning}. For AE, we train simple autoencoders consistent with ~\citep{goldfarb20practical} on MNIST~\citep{mnist}, as well as CURVE and FACES~\citep{curves_faces}. For recommendation, we train H+Vamp Gated~\citep{Kim_2019} on MovieLens20M~\citep{movielens20m}. In the continual learning setting we train on two benchmarks from \citep{lomonaco2021avalanche}: GEM~\citep{GEM} on Permuted MNIST~\citep{permuted_mnist} and LaMAML~\citep{lamaml} on Split CIFAR100~\citep{split_cifar100}.

We choose Shampoo as a baseline 2nd order optimizer as: a.) KrADagrad variants are most similar to Shampoo b.) we expect at best similar performance unless KrADagrad's approximations and lower precision are in practice detrimental. To ground all comparisons, we always include SGD, Adam, plus any unique optimizer from an existing benchmark.

In all our experiments, we initialized from common seeds across optimizers. However, the first points on the training curves follows the number of training steps per evaluation interval, so they do not visually appear to start from the same point. Doing so would have required modifying the codebases separately for each experiment. In addition, ideally we would have had the resources to run multiple shared seeds for the selected HPs of every task we explored. For smaller tasks where convergence is achieved quickly, such as the autoencoder experiments, it was feasible for us to do this and we share this in Figures \ref{fig:ae} and \ref{fig:ae_seeds}. In the tradeoff of how to use our compute and time resources, we opted to present fair evaluations of the optimizers (i.e. ensuring shared seeds and optimized HPs per optimizer per task) across a diversity of tasks, and repeatability statistics from different initial guesses as much as possible.

\setlength{\parskip}{.25\baselineskip}
\subsection{Conditioning Experiment}
To see the effects of Hessian conditioning on optimizers of interest, we create a synthetic deterministic convex problem. For $\X\in\Rbb^{128\times 128}$, we minimize a quadratic loss function\par
{\centering
    $\min_{\X}\, \tr(\X^\top\A\X\B)$
\par}
where $\A,\B \succ 0$ are non-diagonal and have condition numbers $\kappa(\A)=\kappa(\B)=10^{10}$. We seed each optimizer with the same starting point and sweep learning rates and pick the one with the lowest loss for each optimizer. We provide further details in Appendix \ref{app:addl_exp}.
While this stationary problem without a validation or test dataset is different from the online setting assumed in Theorem \ref{thm:krad_regret}, and even from a typical offline machine learning problem, it helps isolate differences in behavior between optimizers in a badly conditioned convex loss landscape.

\setlength{\belowcaptionskip}{-8pt}
\begin{figure}[t]
\includegraphics[width=0.5\textwidth]{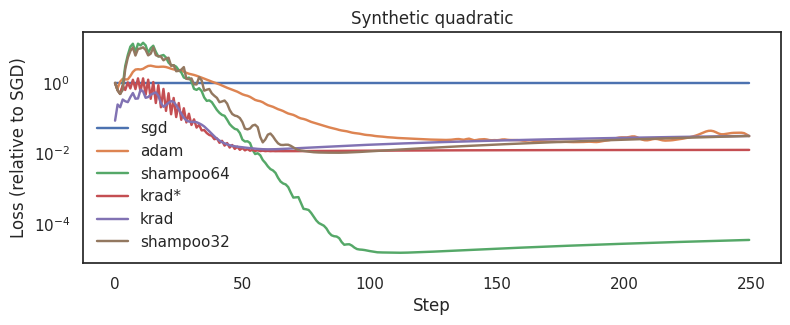}
\caption{Loss on synthetic quadratic in log scale, relative to SGD. Each curve is \textit{divided} by that of SGD.}
\label{fig:quad}
\end{figure}

In Figure \ref{fig:quad}, while Shampoo in double precision outperforms the others in terms of final loss achieved, single precision KrADagrad$^\star$ outperforms single precision Shampoo. Adam does not converge quite as quickly or to as good a loss, as the adaptivity it provides is aligned to the canonical axis, while $\A, \B$ being non-diagonal act in a rotated basis.

\subsection{Real Datasets}

For each task and optimizer we sweep hyperparameters (HP), select those yielding the best validation metric at some epoch or iteration, and display the corresponding learning curves in Figures \ref{fig:resnet} through \ref{fig:hvamp}.  We set the preconditioner update rate for Shampoo and KrADagrad$^\star$ to every 20 training steps. For all experiments, Shampoo uses double precision for negative matrix roots, and SGD includes momentum, unless stated otherwise. We summarize our observations here and provide additional detail in Appendix \ref{app:experiments}.

\begin{figure}[t]
\includegraphics[width=0.5\textwidth]{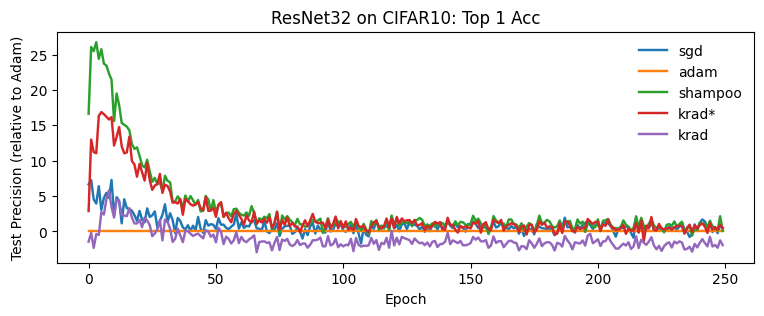}
\includegraphics[width=0.5\textwidth]{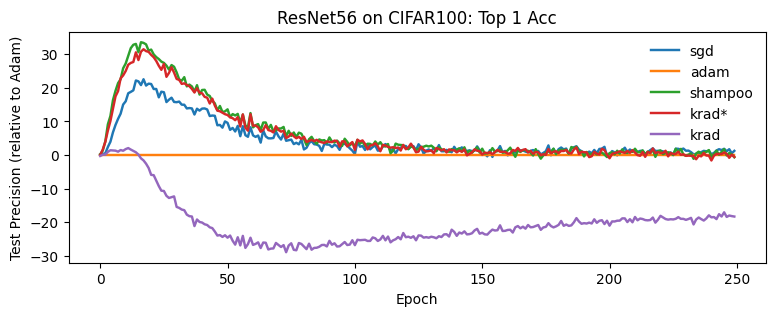}
\caption{Vision experiments. Each curve is initialized from the same single seed and \textit{subtracted} from that of Adam. Top: Top 1 Accuracy (in \%) of ResNet-32 (without BN) on CIFAR-10, relative to Adam; Bottom: Top 1 Accuracy (in \%) of ResNet-56 (without BN) on CIFAR-100, relative to Adam.}
\label{fig:resnet}
\end{figure}
\setlength{\belowcaptionskip}{-6pt}

For the three autoencoder experiments, after HP tuning we retrain the best HPs with 5 unique seeds shared across optimizers, with results in Figures \ref{fig:ae} and \ref{fig:ae_seeds}.  We also find the auto encoder tasks to be relevant baselines for evaluating Shampoo alternatives because Shampoo consistently outperforms SGD and Adam. While KrAdagrad$^\star$ also outperforms SGD and Adam, Shampoo reaches better solutions or similar solutions in fewer steps. Following the synthetic experiments, we test the hypothesis that KrAdagrad$^\star$ might perform better than 32-bit Shampoo by running our best HPs for Shampoo with single precision. The 32- and 64-bit versions perform similarly on these tasks, falsifying the hypothesis in the general sense. KrAdagrad, slower in the number of steps and performing worse in general, eventually reaches similar performance to KrAdagrad$^\star$ on CURVES, with some runs reaching that of 32 bit Shampoo (see Figure \ref{fig:ae_seeds}b). One hypothesis is that on these particular tasks, effectively traversing the loss landscape requires eigenvectors that the KrADagrad approximation has more difficulty capturing than Shampoo.

\begin{figure*}[!h]
    \centering
    \includegraphics[width=0.33\textwidth]{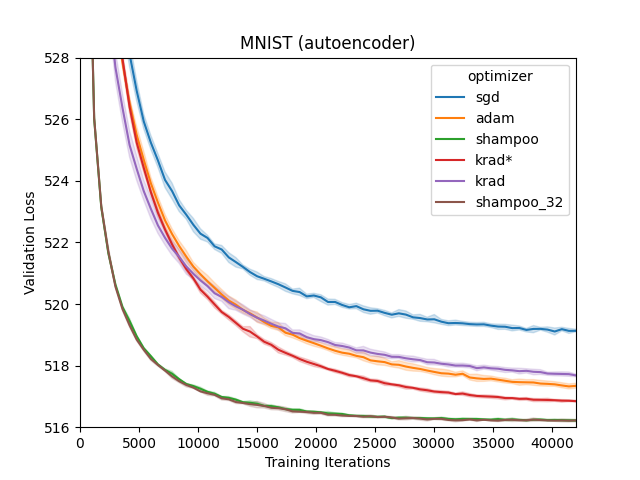}
    \includegraphics[width=0.33\textwidth]{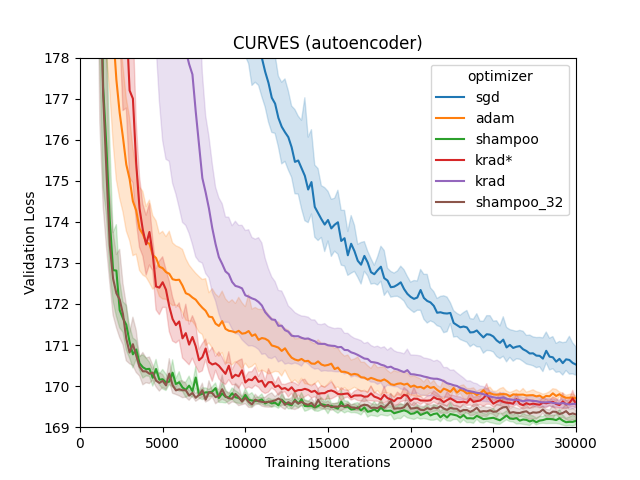}
    \includegraphics[width=0.33\textwidth]{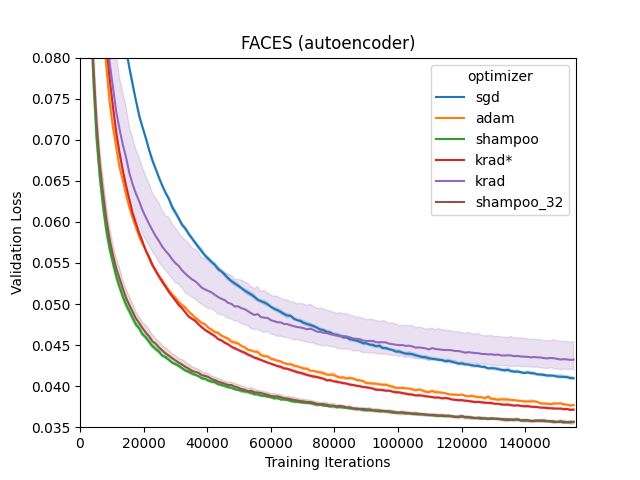}
    \caption{Reconstruction validation error for fully connected auto encoder a) mean cross entropy on MNIST b) mean cross entropy on CURVES c) mean squared error on FACES.  The learning curves shown are averaged across 5 unique seeds with CI95 error bars.  We do this analysis for the autoencoder experiments since they're relatively inexpensive to run.}
    \label{fig:ae}
\end{figure*}




\begin{figure}[t]
    \vspace{-6pt}
    \centering
    \includegraphics[width=0.45\textwidth]{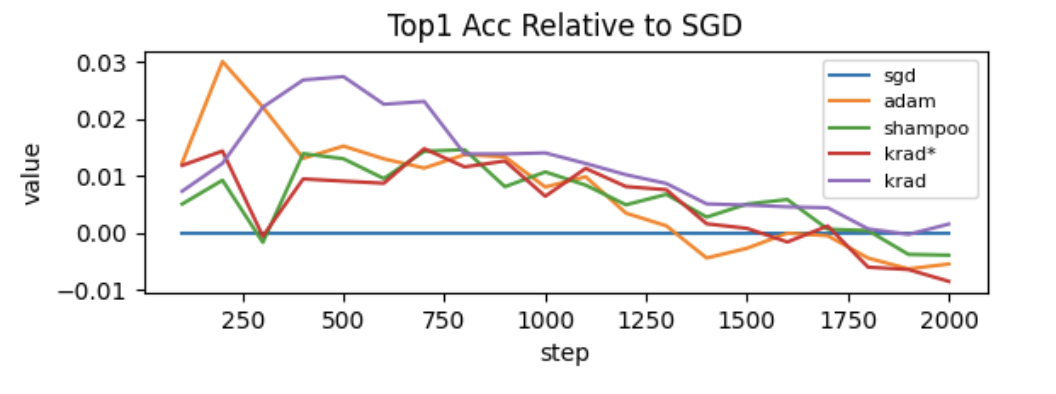}
    \includegraphics[width=0.45\textwidth]{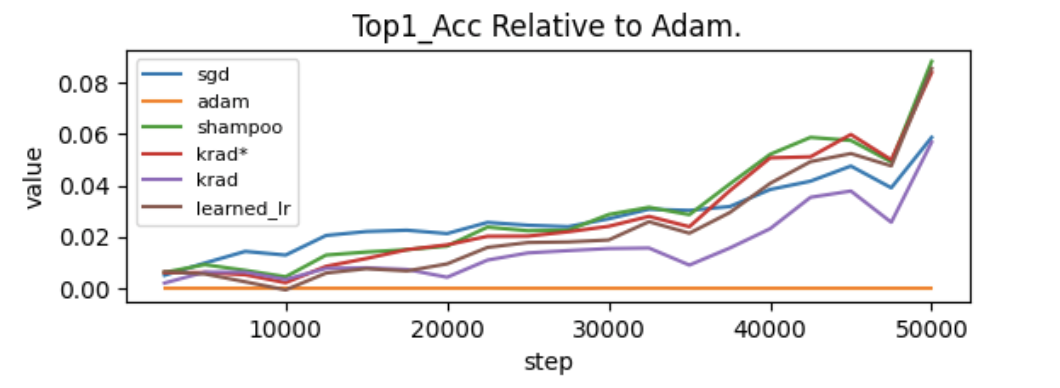}
    \caption{a) Top 1 Accuracy of GEM on Permuted MNIST, relative to SGD. Each curve averages multiple seeds, subtracted from the corresponding SGD curve to illuminate otherwise unobservable differences. See Appendix \ref{app:addl_exp} for absolute curves. b) Top 1 Accuracy of LaMAML on Split CIFAR-100, relative to Adam. }\label{fig:pmnist-lamaml}
    \vspace{-6pt}
\end{figure}

While \cite{gupta18shampoo} performs experiments on ``standard'' machine learning tasks, their regret results pertain to an online setting. We include continual learning problems to test the limits of the theoretical setting. Figure \ref{fig:pmnist-lamaml}a displays learning curves for GEM on Permuted MNIST relative to SGD to illuminate small differences (we include a plot of the actual accuracies in Appendix \ref{app:addl_exp}). On average, KrADagrad stays slightly ahead of others throughout training, unlike KrADagrad$^\star$.


Figure \ref{fig:pmnist-lamaml}b shows similar accuracy curves for LaMAML on Split CIFAR100 relative to Adam (again, the absolute accuracies are in Appendix \ref{app:addl_exp}). KraADagrad$^\star$ closely follows Shampoo and the adaptive learning rate optimizer from \citep{lamaml}, while KrADagrad performs only better than Adam.


We also evaluate Shampoo and KrAdagrad variants against the baseline for H+Vamp Gated\footnote{https://github.com/psywaves/EVCF}. Shampoo reaches a lower training loss compared to KrAdagrad$^\star$, while both Adam variants converge faster and to lower loss values (Figure \ref{fig:hvamp}).

\begin{figure}[t]
\includegraphics[width=0.45\textwidth]{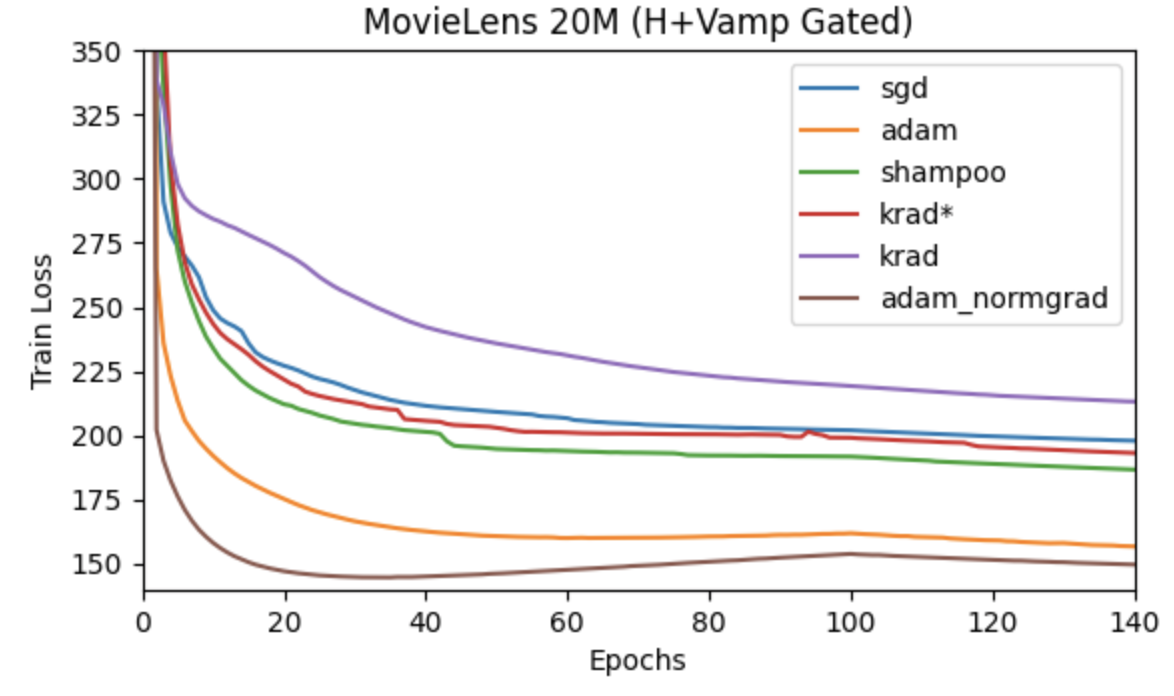}
\caption{Training loss of H+Vamp Gated on ML-20M. We choose to compare the training loss for this experiment because it comprises a weighted combination of components, RE + $\beta\cdot $ KL, where $\beta$ changes according to a predefined schedule (causing loss growth in the 1st 100 epochs), which is not consistent in the validation loss.}
\label{fig:hvamp}
\end{figure}


    \section{Conclusion}
    \label{sec:conclusion}
    \setlength{\parskip}{.5\baselineskip}
We introduced KrADagrad$^\star$, a preconditioned gradient optimizer that avoids matrix inversion, and proved that it has optimal regret properties. We showed that its performance exceeds that of Shampoo in single precision on an ill-conditioned synthetic convex optimization problem. Our experiments on real datasets show that KrADagrad$^\star$ often performs similarly to Shampoo.

Future work includes tractable methods to compute low-rank updates to matrix roots, for instance with rational Krylov subspace methods, to improve the performance of KrADagrad$^\star$ and other Kronecker-factored preconditioners. It is also worth considering the application of the optimizer to other application areas, such as decision making and controls, or scientific discovery.


    \begin{acknowledgements}
        We thank Luminous Computing for funding this work. In particular, we thank David Baker for encouraging this line of investigation and David Scott for fruitful discussions about numerical linear algebra.
    \end{acknowledgements}
\clearpage
\bibliography{main.bib}

\begin{thebibliography}{45}
\providecommand{\natexlab}[1]{#1}
\providecommand{\url}[1]{\texttt{#1}}
\expandafter\ifx\csname urlstyle\endcsname\relax
  \providecommand{\doi}[1]{doi: #1}\else
  \providecommand{\doi}{doi: \begingroup \urlstyle{rm}\Url}\fi

\bibitem[Agarwal et~al.(2019)Agarwal, Bullins, Chen, Hazan, Singh, Zhang, and
  Zhang]{agarwal19efficient}
Naman Agarwal, Brian Bullins, Xinyi Chen, Elad Hazan, Karan Singh, Cyril Zhang,
  and Yi~Zhang.
\newblock Efficient full-matrix adaptive regularization.
\newblock In \emph{International Conference on Machine Learning}, pages
  102--110. PMLR, 2019.

\bibitem[Ando et~al.(2004)Ando, Li, and Mathias]{ando04geometric}
T.~Ando, Chi~Kwong Li, and Roy Mathias.
\newblock Geometric means.
\newblock \emph{Linear Algebra and Its Applications}, 385:\penalty0 305--334, 7
  2004.
\newblock ISSN 00243795.
\newblock \doi{10.1016/J.LAA.2003.11.019}.

\bibitem[Anil et~al.(2020)Anil, Gupta, Koren, Regan, and
  Singer]{anil20scalable}
Rohan Anil, Vineet Gupta, Tomer Koren, Kevin Regan, and Yoram Singer.
\newblock Scalable second order optimization for deep learning.
\newblock \emph{arXiv preprint arXiv:2002.09018}, 2020.

\bibitem[Anil et~al.(2022)Anil, Gadanho, Huang, Jacob, Li, Lin, Phillips, Pop,
  Regan, Shamir, Shivanna, and Yan]{anil22factory}
Rohan Anil, Sandra Gadanho, Da~Huang, Nijith Jacob, Zhuoshu Li, Dong Lin, Todd
  Phillips, Cristina Pop, Kevin Regan, Gil~I. Shamir, Rakesh Shivanna, and Qiqi
  Yan.
\newblock On the factory floor: {ML} engineering for industrial-scale ads
  recommendation models.
\newblock ACM RecSys, 9 2022.
\newblock \doi{10.48550/arxiv.2209.05310}.
\newblock URL
  \url{https://orsum.inesctec.pt/orsum2022/assets/files/paper6.pdf}.

\bibitem[Baumgartner(2011)]{baumgartner11inequality}
Bernhard Baumgartner.
\newblock An inequality for the trace of matrix products, using absolute
  values.
\newblock 6 2011.
\newblock \doi{10.48550/arxiv.1106.6189}.

\bibitem[Bellman(1980)]{bellman80some}
Richard Bellman.
\newblock Some inequalities for positive definite matrices.
\newblock In \emph{General Inequalities 2}, pages 89--90. Springer, 1980.

\bibitem[Bhatia(2009)]{bhatia09positive}
Rajendra Bhatia.
\newblock \emph{Positive Definite Matrices}.
\newblock Princeton University Press, 12 2009.
\newblock ISBN 9781400827787.
\newblock \doi{10.1515/9781400827787}.

\bibitem[Boyd and Vandenberghe(2004)]{boyd04convex}
Stephen~P Boyd and Lieven Vandenberghe.
\newblock \emph{Convex optimization}.
\newblock Cambridge university press, 2004.

\bibitem[Dongarra et~al.(2014)Dongarra, Gates, Haidar, Kurzak, Luszczek, Tomov,
  and Yamazaki]{dongarra14accelerating}
Jack Dongarra, Mark Gates, Azzam Haidar, Jakub Kurzak, Piotr Luszczek,
  Stanimire Tomov, and Ichitaro Yamazaki.
\newblock Accelerating numerical dense linear algebra calculations with gpus.
\newblock \emph{Numerical Computations with {GPUs}}, pages 1--26, 2014.

\bibitem[Duchi et~al.(2011)Duchi, Hazan, and Singer]{duchi11adaptive}
John Duchi, Elad Hazan, and Yoram Singer.
\newblock Adaptive subgradient methods for online learning and stochastic
  optimization.
\newblock \emph{Journal of Machine Learning Research}, 12\penalty0 (7), 2011.

\bibitem[Frangella et~al.(2022)Frangella, Rathore, Zhao, and
  Udell]{frangella22sketchysgd}
Zachary Frangella, Pratik Rathore, Shipu Zhao, and Madeleine Udell.
\newblock Sketchy{SGD}: Reliable stochastic optimization via robust curvature
  estimates.
\newblock 11 2022.
\newblock \doi{10.48550/arxiv.2211.08597}.
\newblock URL \url{https://arxiv.org/abs/2211.08597v2}.

\bibitem[Goldfarb et~al.(2020)Goldfarb, Ren, and Bahamou]{goldfarb20practical}
Donald Goldfarb, Yi~Ren, and Achraf Bahamou.
\newblock Practical quasi-newton methods for training deep neural networks.
\newblock In \emph{Advances in Neural Information Processing Systems}, pages
  2386--2396, 2020.

\bibitem[Goodfellow et~al.(2013)Goodfellow, Mirza, Xiao, Courville, and
  Bengio]{permuted_mnist}
Ian~J. Goodfellow, Mehdi Mirza, Da~Xiao, Aaron Courville, and Yoshua Bengio.
\newblock An empirical investigation of catastrophic forgetting in
  gradient-based neural networks, 2013.
\newblock URL \url{https://arxiv.org/abs/1312.6211}.

\bibitem[Gupta et~al.(2020)Gupta, Yadav, and Paull]{lamaml}
Gunshi Gupta, Karmesh Yadav, and Liam Paull.
\newblock Look-ahead meta learning for continual learning.
\newblock In H.~Larochelle, M.~Ranzato, R.~Hadsell, M.F. Balcan, and H.~Lin,
  editors, \emph{Advances in Neural Information Processing Systems}, volume~33,
  pages 11588--11598. Curran Associates, Inc., 2020.

\bibitem[Gupta et~al.(2018)Gupta, Koren, and Singer]{gupta18shampoo}
Vineet Gupta, Tomer Koren, and Yoram Singer.
\newblock Shampoo: Preconditioned stochastic tensor optimization.
\newblock In \emph{International Conference on Machine Learning}, pages
  1842--1850. PMLR, 2018.

\bibitem[Harper and Konstan(2015)]{movielens20m}
F.~Maxwell Harper and Joseph~A. Konstan.
\newblock The {MovieLens} datasets: History and context.
\newblock \emph{ACM Trans. Interact. Intell. Syst.}, 5\penalty0 (4), dec 2015.
\newblock ISSN 2160-6455.
\newblock \doi{10.1145/2827872}.

\bibitem[He et~al.(2016)He, Zhang, Ren, and Sun]{he16deep}
Kaiming He, Xiangyu Zhang, Shaoqing Ren, and Jian Sun.
\newblock Deep residual learning for image recognition.
\newblock In \emph{Proceedings of the IEEE conference on computer vision and
  pattern recognition}, pages 770--778, 2016.
\newblock URL \url{http://image-net.org/challenges/LSVRC/2015/}.

\bibitem[Higham(2008)]{higham08functions}
Nicholas~J. Higham.
\newblock \emph{Functions of Matrices: Theory and Computation}.
\newblock Society for Industrial and Applied Mathematics, 2008.
\newblock ISBN 978-0-898716-46-7.

\bibitem[Hinton and Salakhutdinov(2006)]{curves_faces}
G.~E. Hinton and R.~R. Salakhutdinov.
\newblock Reducing the dimensionality of data with neural networks.
\newblock \emph{Science}, 313\penalty0 (5786):\penalty0 504--507, 2006.
\newblock \doi{10.1126/science.1127647}.
\newblock URL \url{https://www.science.org/doi/abs/10.1126/science.1127647}.

\bibitem[Idelbayev()]{idelbayev18aproper}
Yerlan Idelbayev.
\newblock Proper {ResNet} implementation for {CIFAR10/CIFAR100} in {PyTorch}.
\newblock \url{https://github.com/akamaster/pytorch_resnet_cifar10}.
\newblock Accessed: 2022-12-12.

\bibitem[Jouppi et~al.(2021)Jouppi, Yoon, Ashcraft, Gottscho, Jablin, Kurian,
  Laudon, Li, Ma, Ma, Norrie, Patil, Prasad, Young, Zhou, and
  Patterson]{jouppi21ten}
Norman~P. Jouppi, Doe~Hyun Yoon, Matthew Ashcraft, Mark Gottscho, Thomas~B.
  Jablin, George Kurian, James Laudon, Sheng Li, Peter Ma, Xiaoyu Ma, Thomas
  Norrie, Nishant Patil, Sushma Prasad, Cliff Young, Zongwei Zhou, and David
  Patterson.
\newblock {Ten Lessons From Three Generations Shaped Google's TPUv4i Industrial
  Product}.
\newblock 2021.
\newblock \doi{10.1109/ISCA52012.2021.00010}.

\bibitem[Kim and Suh(2019)]{Kim_2019}
Daeryong Kim and Bongwon Suh.
\newblock Enhancing {VAEs} for collaborative filtering.
\newblock In \emph{Proceedings of the 13th {ACM} Conference on Recommender
  Systems}. {ACM}, sep 2019.
\newblock \doi{10.1145/3298689.3347015}.

\bibitem[Kingma and Ba(2015)]{kingma14adam}
Diederik~P. Kingma and Jimmy Ba.
\newblock Adam: A method for stochastic optimization.
\newblock In \emph{ICLR (Poster)}, 2015.

\bibitem[Krizhevsky(2009)]{krizhevsky09learning}
Alex Krizhevsky.
\newblock Learning multiple layers of features from tiny images.
\newblock Technical report, 2009.

\bibitem[Lecun et~al.(1998)Lecun, Bottou, Bengio, and Haffner]{mnist}
Y.~Lecun, L.~Bottou, Y.~Bengio, and P.~Haffner.
\newblock Gradient-based learning applied to document recognition.
\newblock \emph{Proceedings of the IEEE}, 86\penalty0 (11):\penalty0
  2278--2324, 1998.
\newblock \doi{10.1109/5.726791}.

\bibitem[Lomonaco et~al.(2021)Lomonaco, Pellegrini, Cossu, Carta, Graffieti,
  Hayes, Lange, Masana, Pomponi, van~de Ven, Mundt, She, Cooper, Forest,
  Belouadah, Calderara, Parisi, Cuzzolin, Tolias, Scardapane, Antiga, Amhad,
  Popescu, Kanan, van~de Weijer, Tuytelaars, Bacciu, and
  Maltoni]{lomonaco2021avalanche}
Vincenzo Lomonaco, Lorenzo Pellegrini, Andrea Cossu, Antonio Carta, Gabriele
  Graffieti, Tyler~L. Hayes, Matthias~De Lange, Marc Masana, Jary Pomponi, Gido
  van~de Ven, Martin Mundt, Qi~She, Keiland Cooper, Jeremy Forest, Eden
  Belouadah, Simone Calderara, German~I. Parisi, Fabio Cuzzolin, Andreas
  Tolias, Simone Scardapane, Luca Antiga, Subutai Amhad, Adrian Popescu,
  Christopher Kanan, Joost van~de Weijer, Tinne Tuytelaars, Davide Bacciu, and
  Davide Maltoni.
\newblock Avalanche: an end-to-end library for continual learning.
\newblock In \emph{Proceedings of IEEE Conference on Computer Vision and
  Pattern Recognition}, 2nd Continual Learning in Computer Vision Workshop,
  2021.
\newblock URL
  \url{https://github.com/ContinualAI/continual-learning-baselines}.

\bibitem[Lopez-Paz and Ranzato(2017)]{GEM}
David Lopez-Paz and Marc\textquotesingle~Aurelio Ranzato.
\newblock Gradient episodic memory for continual learning.
\newblock In I.~Guyon, U.~Von Luxburg, S.~Bengio, H.~Wallach, R.~Fergus,
  S.~Vishwanathan, and R.~Garnett, editors, \emph{Advances in Neural
  Information Processing Systems}, volume~30. Curran Associates, Inc., 2017.

\bibitem[Lu et~al.(2013)Lu, Plataniotis, and Venetsanopoulos]{lu13multilinear}
Haiping Lu, Konstantinos~N. Plataniotis, and Anastasios~N. Venetsanopoulos.
\newblock \emph{Multilinear subspace learning: Dimensionality reduction of
  multidimensional data}.
\newblock CRC Press, 2013.
\newblock ISBN 9781439857298.
\newblock \doi{10.1201/B16252}.

\bibitem[Martens(2020)]{martens20new}
James Martens.
\newblock New insights and perspectives on the natural gradient method.
\newblock \emph{Journal of Machine Learning Research}, 21:\penalty0 1--76,
  2020.
\newblock ISSN 1533-7928.

\bibitem[Martens and Grosse(2015)]{martens15optimizing}
James Martens and Roger Grosse.
\newblock Optimizing neural networks with kronecker-factored approximate
  curvature.
\newblock In \emph{International Conference on Machine Learning}, pages
  2408--2417, 2015.

\bibitem[{NVidia Team}(2023)]{cusolver23api}
{NVidia Team}.
\newblock {cuSOLVER API Reference}, January 2023.
\newblock URL \url{https://docs.nvidia.com/cuda/cusolver}.

\bibitem[Paszke et~al.(2017)Paszke, Gross, Chintala, Chanan, Yang, DeVito, Lin,
  Desmaison, Antiga, and Lerer]{paszke17automatic}
Adam Paszke, Sam Gross, Soumith Chintala, Gregory Chanan, Edward Yang, Zachary
  DeVito, Zeming Lin, Alban Desmaison, Luca Antiga, and Adam Lerer.
\newblock Automatic differentiation in {p}ytorch.
\newblock 2017.

\bibitem[{PyTorch Team}(2021)]{pytorch21blog}
{PyTorch Team}.
\newblock The torch.linalg module: Accelerated linear algebra with autograd in
  {P}y{T}orch | {P}y{T}orch, June 2021.
\newblock URL \url{https://pytorch.org/blog/torch-linalg-autograd/}.

\bibitem[Reddi et~al.(2018)Reddi, Kale, and Kumar]{reddi18convergence}
Sashank~J Reddi, Satyen Kale, and Sanjiv Kumar.
\newblock On the convergence of {A}dam and beyond.
\newblock \emph{International Conference on Learning Representations}, 2018.

\bibitem[Ren and Goldfarb(2021)]{ren21tensor}
Yi~Ren and Donald Goldfarb.
\newblock Tensor normal training for deep learning models.
\newblock \emph{Advances in Neural Information Processing Systems},
  34:\penalty0 26040--26052, 12 2021.

\bibitem[Shukla(2022)]{shukla22understanding}
Lavanya Shukla.
\newblock {Understanding the Landscape of the latest Large Models}.
\newblock \emph{Advances in neural information processing systems}, 35, 2022.

\bibitem[Song et~al.(2021)Song, Sebe, and Wang]{song21approximate}
Yue Song, Nicu Sebe, and Wei Wang.
\newblock Why approximate matrix square root outperforms accurate {SVD} in
  global covariance pooling?
\newblock In \emph{Proceedings of the IEEE/CVF International Conference on
  Computer Vision}, pages 1115--1123, 2021.

\bibitem[Srebro et~al.(2011)Srebro, Sridharan, and
  Tewari]{srebro11universality}
Nati Srebro, Karthik Sridharan, and Ambuj Tewari.
\newblock On the universality of online mirror descent.
\newblock \emph{Advances in neural information processing systems}, 24, 2011.

\bibitem[Tseng(2001)]{tseng01convergence}
Paul Tseng.
\newblock Convergence of a block coordinate descent method for
  nondifferentiable minimization.
\newblock \emph{Journal of Optimization Theory and Applications 2001 109:3},
  109:\penalty0 475--494, 6 2001.
\newblock ISSN 1573-2878.
\newblock \doi{10.1023/A:1017501703105}.

\bibitem[{Van Loan}(2000)]{vanloan00ubiquitous}
Charles~F. {Van Loan}.
\newblock The ubiquitous kronecker product.
\newblock \emph{Journal of Computational and Applied Mathematics},
  123:\penalty0 85--100, 11 2000.
\newblock ISSN 03770427.
\newblock \doi{10.1016/S0377-0427(00)00393-9}.

\bibitem[{Van Loan} and Pitsianis(1993)]{vanloan93approximation}
Charles~F. {Van Loan} and Nikos Pitsianis.
\newblock Approximation with kronecker products.
\newblock In \emph{Linear Algebra for Large Scale and Real-Time Applications},
  pages 293--314. Springer, Dordrecht, 1993.
\newblock \doi{10.1007/978-94-015-8196-7_17}.

\bibitem[Woodbury(1950)]{woodbury50inverting}
Max~A Woodbury.
\newblock \emph{Inverting modified matrices}.
\newblock Statistical Research Group, 1950.

\bibitem[Yang et~al.(2022)Yang, Xu, Wen, Chen, and Xu]{yang22sketch}
Minghan Yang, Dong Xu, Zaiwen Wen, Mengyun Chen, and Pengxiang Xu.
\newblock Sketch-based empirical natural gradient methods for deep learning.
\newblock \emph{Journal of Scientific Computing}, 92:\penalty0 1--29, 9 2022.
\newblock ISSN 15737691.
\newblock \doi{10.1007/S10915-022-01911-X/METRICS}.

\bibitem[Yao et~al.(2021)Yao, Gholami, Shen, Mustafa, Keutzer, and
  Mahoney]{yao21adahessian}
Zhewei Yao, Amir Gholami, Sheng Shen, Mustafa Mustafa, Kurt Keutzer, and
  Michael~W Mahoney.
\newblock Ada{H}essian: An adaptive second order optimizer for machine
  learning.
\newblock 2021.

\bibitem[Zenke et~al.(2017)Zenke, Poole, and Ganguli]{split_cifar100}
Friedemann Zenke, Ben Poole, and Surya Ganguli.
\newblock Improved multitask learning through synaptic intelligence.
\newblock \emph{CoRR}, abs/1703.04200, 2017.
\newblock URL \url{http://arxiv.org/abs/1703.04200}.

\end{thebibliography}

\title{KrADagrad: Kronecker Approximation-Domination Gradient\\Preconditioned Stochastic Optimization\\(Supplementary Material)}
\onecolumn
\maketitle

\appendix

    \startcontents[appendices]
    \printcontents[appendices]{l}{1}{\setcounter{tocdepth}{2}}
    
    \section{KrADagrad}
    \label{app:krad}
    \subsection{KrADagrad: direct Adagrad-style alternating updates}
\label{subsec:KrADagrad_adagrad}
    We propose an alternative scheme for updating the statistics. We alternate between directly using KrAD to update left and right statistics in different iterations. This has the practical advantage of only requiring one matrix  statistic update per iteration and only needing a matrix square root (even in the extension to tensors!). However, the scheme only attains optimal regret when we stop updating the right preconditioner after a fixed number of iterations, though we note that we describe how to choose an update schedule to obtain regret arbitrarily close to optimal.
    
    Arguably, the most intuitive way to update $(\L_k,\R_k)$ is to select an index set $\Ical\subset \mathbb{N}$ as a subset of the natural numbers, so that if $k\in\Ical$ we update $\L_k$ and hold $\R_k$ constant, and vice versa for $k\not\in\Ical$. 
    
    It is straightforward to see that we simply need to compute a matrix square root on KrAD statistics to apply Adagrad-style preconditioning,
    \begin{align}
    \label{eq:KrADagrad_Adagrad_update}
        \W_{k+1} &= \W_k - \eta_k \L_k^{1/2}\G_k \R_k^{1/2}.
    \end{align}
    In an iteration $k\in\Ical$, since we have held $\R_k$ constant, we can store its square root from the previous iteration and only need to compute a square root for $\L_k$ that we updated in this iteration. In this case, we refer to our intermediate matrices $(\L_k,\R_k)$ as our statistics and the pair $(\L_k^{1/2}, \R_k^{1/2})$ as the actual preconditioners. Thus, we can see that our storage requirement is now $(\L_k,\L_k^{1/2}, \R_k, \R_k^{1/2})$.

    We call this update scheme KrADagrad and summarize in Appendix~\ref{app:implementation_kradagrad}.
    
    \subsubsection{Regret}
        We establish some results for these updates to provide the groundwork for the theoretical properties of the algorithm derived from the updates. 
        \begin{lemma}
        \label{lem:Q_dom_GGT}
            The corresponding $\Q_k\succeq \frac{1}{nt}\Gcal_k\Gcal_k^\top$, where
            $t=\max_k t_k$.
        \end{lemma}
        \begin{proof}
            Taking $\Delta\L_k = \frac{1}{t_k}\L_{k\!-\!1}\G_k\R_{k}\G_k^{\top}\L_{k\!-\!1}$ and $\L_k=\L_{k-1}-\Delta\L_k$, we have from the chain of inequalities preceding Equation~\eqref{eq:krad_update_dominate_inverse},
            \begin{align*}
                &\L_{k}^{-1}=(\L_{k-1}-\Delta\L_k)^{-1} \succeq \B_{k-1}+\Delta\B_{k} \\
                \Rightarrow& \Q_k = \R_{k}^{-1} \otimes \L_{k}^{-1} \succeq \C_{k} \otimes \widetilde\B_{k} \succeq \frac{1}{nt}\Gcal_k\Gcal_k^\top
            \end{align*}
            where $\L_{k}$ is invertible due to it being PD as shown in Corollary~\ref{cor:KrAD_update_PD}, and the last domination inequality is due to Equation~\eqref{eqn:b-tilde-dominating-scaled-gradients}. We can see that $t=\max_{k} t_k$ suffices.
        \end{proof}
        Lemma~\ref{lem:Q_dom_GGT} is half of what we need to show good regret guarantees, as per proof of Theorem 7 in~\cite{gupta18shampoo}.
        The other half is the Lemma \ref{lem:trace_bound} from the main text.
        
        Thus, we have the next result, which plugs in the trace bounds of Lemma~\ref{lem:trace_bound} into Lemma~\ref{lem:gupta}.
        \begin{theorem}
            Over $K$ iterations, assume we update $\R_k$ in $K'=o(K)$ iterations and $\L_k$ in the remaining $K-K'$ iterations. The regret from using KrADagrad with Adagrad-style $\alpha=1/2$ exponent scales as\par
            {\centering $R(K)=O(\sqrt{KK'})$ \par}
        \end{theorem}
        \begin{proof}
            With the modified form of the H\"{o}lder inequality for PSD matrices from~\pref{property:trace_holder}, 
            \begin{align*}
                \tr(\A^{1/2}\B^{1/2})& \le \tr(\A)^{1/2}\tr(\B)^{1/2} \\
                \Rightarrow\tr(\B_{k}^{1/2})& \le m^{1/2}\tr(\B_{k})^{1/2}.
            \end{align*}
            Then, from Equation~\eqref{eq:trace1},
            \begin{align*}
                \tr(\B_{k}^{1/2}) &\le m^{1/2}\tr(\B_{k})^{1/2}\\
                &\le \sqrt{m\tr(\B_{0}) \!+\! k'm^2 s\|\R_0\|_2}.
            \end{align*}
            Analogously, from Equation~\eqref{eq:trace2}
            \begin{align*}
                \tr(\C_{k}^{1/2}) &\le n^{1/2}\tr(\C_{k})^{1/2}\\
                &\le \sqrt{n\tr(\C_{0}) \!+\! (k \!-\! k')n^2 s\|\L_0\|_2}.
            \end{align*}
            Assuming w.l.o.g. that $k'=o(k)$, the regret bounded by $\tr(\B_{k}^{1/2})\tr(\C_{k}^{1/2})$ is further bounded by
            \begin{align*}
                \sqrt{(am+m^2 b(k \!-\! k'))(cn \!+\! n^2 dk')}=O(\sqrt{kk'}),
            \end{align*}
            where $a,b,c,d$ are constants w.r.t. $m,n,k,k'$ that can be expressed in terms of $\tr(\B_0)$, $\tr(\C_0)$, $\|\L_0\|_2$, and $\|\R_0\|_2$.
            Finally, $t_k = 1+\|\L_{k-1}\G_k\R_{k-1}\G_k^\top\|_2\le s$ for all $k>0$, so $t = \max_k t_k \le s$.
        \end{proof}
    We see that this is greater than the optimal rate of $O(\sqrt{K})$ unless $K'=O(1)$. Taking an update scheme in which $K'\sim K^{2\varepsilon}$, we can achieve a rate of $O(K^{1/2+\varepsilon})$, which can be made arbitrarily close to the optimal. In one sense, keeping one Kronecker factor as identity (as Shampoo does) can be seen as optimal according to our regret result, as it does not introduce the additional factor of $k'$ in the bound. This inspires an alternative way of updating $\R_k$, resulting in KrADagrad$^\star$ from the main text.
    
    \section{Deferred proofs}
    \label{app:proofs}
    Here we provide proofs that were either too long for the main text or pertained to less consequential results.

\subsection{Proof of Lemma~\ref{lem:krad}}
\label{app:proof_krad}
\krad*
\begin{proof}
        We largely follow~\cite{gupta18shampoo}. We consider the case where $r=1$ and drop the subscript on $\u_i$ for clarity (i.e. $\U=\U_1$) since summation holds across the inequality.
        Let the singular value decomposition (SVD) of $\R=\V\boldsymbol{\Lambda}\V^\top$, and $\v_j$ be the $j$-th column of $\V$ and $\lambda_j=\boldsymbol{\Lambda}_{j,j}$ the corresponding singular value. Let $\Z=\U\V$. Then
        \begin{align*}
            \U&=\Z\V^\top\\
            \L&=\U\R^{-1}\U^{\top}=\Z\boldsymbol{\Lambda}^{-1}\Z^\top.
        \end{align*}
        Note that $\Z$ does not necessarily have any special structure, since it was not chosen according to any factorization of $\U$. Let $\z_j$ be the $j$-th column of $\Z$. Then 
        \begin{align*}
            \U &= \sum_{i=1}^{n}\z_i\v_i^\top\\
            \Rightarrow \u &= \sum_{i=1}^{n}\v_i \otimes \z_i. &\because\pref{property:kron_vec}
        \end{align*}
        Next, we note that for a set of vectors $\{\y_i\}$ and any arbitrary vector $\x$, using equivalence of norms
        \begin{align*}
            \x^\top \left(\sum_{i=1}^{n}\y_i\right) \left(\sum_{i=1}^{n}\y_i^\top\right) \x &= \left(\sum_{i=1}^{n}\y_i^\top\x\right)^2\\
            &\le \left(\sum_{i=1}^{n}|\y_i^\top \x| \right)^2\\
            &\le n\sum_{i=1}^{n}\left(\y_i^\top \x\right)^2\\
            &= n\x^\top\left(\sum_{i=1}^{n}\y_i\y_i^\top \right)\x\\
            \Rightarrow \left(\sum_{i=1}^{n}\y_i\right) \left(\sum_{i=1}^{n}\y_i^\top\right) &\preceq n\sum_{i=1}^{n}\y_i\y_i^\top.
        \end{align*}
        Then consider
        \begin{align}
            \u\u^\top &= \left(\sum_{i=1}^{n}\v_i \otimes \z_i\right) \left(\sum_{i=1}^{n}\v_i \otimes \z_i\right)^\top \\
                &\preceq n\sum_{i=1}^{n}(\v_i \otimes \z_i)(\v_i \otimes \z_i)^\top \\
                &=n\sum_{i=1}^{n}(\v_i\v_i^\top \otimes \z_i \z_i^\top) & \because\pref{property:kron_prod}\\
                &=n\sum_{i=1}^{n}(\lambda_i\v_i\v_i^\top \otimes \lambda_i^{-1}\z_i \z_i^\top)\\
                &\preceq n\sum_{i=1}^{n}(\R \otimes \lambda_i^{-1}\z_i \z_i^\top)\label{eqn:R-dominates}\\
                &=n\R \otimes (\Z \boldsymbol{\Lambda}^{-1} \Z^\top)\\
                &=n\R \otimes \L.
        \end{align}
        Here Equation~\ref{eqn:R-dominates} holds since for any $i$,
        \begin{align*}
            \R - \lambda_i\v_i\v_i^\top &= \sum_{j\ne i} \lambda_j  \v_j \v_j^\top \succeq 0.
        \end{align*}
    \end{proof}

\subsection{Proof of Proposition~\ref{prop:middle_term}}
\label{app:proof_middle_term}
    First, we prove our intermediate result~\ref{prop:middle_term}, which we repeat here for convenience.
    \propmiddleterm*
    \begin{proof}
        Since $\L_{k-1}^{1/2}\succ 0$, it is invertible. Thus the matrices $t_k \M_k=\L_{k-1}^{1/2}\G_{k}\R_{k-1}\G_{k}^\top \L_{k-1}^{1/2} $ and $\L_{k-1}\G_k\R_{k-1}\G_{k}^\top$ are similar as
        \begin{align*}
            \L_{k-1}^{1/2} t_k \M_k\L_{k-1}^{-1/2}=\L_{k-1}\G_k\R_{k-1}\G_{k}^\top.
        \end{align*}
        This implies that these matrices have the same matrix norms
        \[
            \|t_k \M_k\|_2 = \|\L_{k-1}\G_k\R_{k-1}\G_{k}^\top\|_2.
        \]
        Next,
        \begin{align*}
            t_k &\ge 1 + \|\L_{k-1}\G_k\R_{k-1}\G_{k}^\top\|_2.
        \end{align*}
        Finally,
        \begin{align*}
            \left\|\M_k\right\|_2 &= \frac{1}{t_k}\|\L_{k-1}\G_k\R_{k-1}\G_{k}^\top\|_2 < 1.
        \end{align*}
    \end{proof}
    
    \subsubsection{Proof of Corollary~\ref{cor:KrAD_update_PD}}
        \label{app:proof_KrAD_update_PD}
        \corKrADupdatePD*
        
        \begin{proof}
            It is enough to consider the matrix
            \begin{align*}
                \L_k &= \L_{k-1}-\frac{1}{t_k}\L_{k-1}\G_{k}\R_{k-1}\G_{k}^\top \L_{k-1}\\
                &= \L_{k-1}^{1/2}(\I-\M_k)\L_{k-1}^{1/2}
            \end{align*}
            so that by \pref{property:domination_quadratic} we only need to examine the middle factor
            \begin{align*}
                \I-\M_k.
            \end{align*}
            If this matrix is positive definite, then $\L_k \succ 0$. Since we have already proven this in Proposition~\ref{prop:middle_term}, we are done.
        \end{proof}

\subsection{Proof of Lemma~\ref{lem:trace_bound}}
    \lemtracebound*
    \begin{proof}
        Suppose w.l.o.g. that $k$ is an iteration in which we have updated $\L_k$. Let $\widetilde\N_k=\left(\G_{k} \R_{k\!-\!1}\G_{k}^\top\right)^{1/2}$, $\N_k=\frac{1}{t_k^{1/2}}\widetilde\N_k$, $\T_{k}=\N_k\L_{k-1}\N_k$, and $\widetilde\T_{k}=\widetilde\N_k\L_{k-1}\widetilde\N_k$. Now, $\T_k =\frac{1}{t_k}\widetilde\T_{k}$ and $t_k\ge 1+\|\widetilde\T_{k}\|_2$. First note the following generalized binomial expansion (or equivalently, the infinite geometric series),
        \begin{align*}
            (\I-\M_k)^{-1} = \sum_{i=0}^\infty \M_k^i,
        \end{align*}
        which converges if $\Vert\M_k\Vert_2< \1$. Since we have shown this in our proof of Proposition~\ref{prop:middle_term}, we may use the series expansion.
        Now note that
        \begin{align*}
            &\B_{k} =[\L_{k\!-\!1}^{1/2}(\I \!-\! \M_k)\L_{k\!-\!1}^{1/2}]^{-1}\\
            &= \L_{k\!-\!1}^{-1/2}(\I \!-\! \M_k)^{-1}\L_{k\!-\!1}^{-1/2}\\
            &= \L_{k\!-\!1}^{-1/2}\sum_{i=0}^{\infty}\M_k^{i}\L_{k\!-\!1}^{-1/2} & \because\|\M_k\|<1\\
            &= \L_{k\!-\!1}^{-1} + \frac{1}{t_k}\G_{k} \R_{k\!-\!1}\G_{k}^\top+\! \frac{1}{t_k^2}\G_{k} \R_{k\!-\!1}\G_{k}^\top\L_{k\!-\!1}\G_{k} \R_{k\!-\!1}\G_{k}^\top \!+\! \ldots \\
            &= \B_{k\!-\!1} + \N_k\sum_{i=0}^{\infty}( \N_k\L_{k\!-\!1}\N_k)^{i}\N_k,
        \end{align*}
        Then
        \begin{align*}
            &\tr(\B_{k}) \!=\! \tr(\B_{k\!-\!1}) \!+\! \tr(\N_{k}\sum_{i=0}^{\infty} \T_k^{i}\N_k)\\
            \Rightarrow &\tr(\B_{k}) \!-\! \tr(\B_{k\!-\!1})\le \|\N_k^2\|_2\sum_{i=0}^{\infty} \tr\left(\T_k^{i} \right)\\
            &\le \|\N_k^2\|_2(m+\sum_{i=1}^{\infty} \|\T_k\|_2^{i-1}\tr(\T_k)) & \because \pref{property:trace_power}\\
            &=\|\N_k^2\|_2\left(m+\frac{\tr(\widetilde\T_k)}{\|\widetilde\T_k\|_2}\sum_{i=1}^{\infty} \|\T_k\|_2^{i}\right)\\
            &=\|\N_k^2\|_2\left(m+\frac{\tr(\widetilde\T_k)}{\|\widetilde\T_k\|_2}\sum_{i=1}^{\infty} \left(\frac{\|\widetilde\T_k\|_2}{t_k}\right)^{i}\right)\\
            &\le \|\N_k^2\|_2\left(m+\frac{\tr(\widetilde\T_k)}{\|\widetilde\T_k\|_2}\sum_{i=1}^{\infty} \left(\frac{\|\widetilde\T_k\|_2}{1+\|\widetilde\T_k\|_2} \right)^{i}\right) & \because t_k\ge 1+\|\widetilde\T_{k}\|_2\\
            &= \|\N_k^2\|_2\left(m+\frac{\tr(\widetilde\T_k)}{\|\widetilde\T_k\|_2}\sum_{i=1}^{\infty} \left(\frac{\|\widetilde\T_k\|_2}{1+\|\widetilde\T_k\|_2} \right)^{i}\right)\\
            &= \|\N_k^2\|_2\left(m+\tr(\widetilde\N_k\L_{k-1}\widetilde\N_k)\right),
        \end{align*}
        where the last line follows by evaluating the convergent geometric series. Applying 1-Lipschitz loss, we have
        \begin{align*}
            \tr(\widetilde\N_k\L_{k-1}\widetilde\N_k) &\le \|\L_{k-1}\|_2\tr(\widetilde\N_k^2) \because \pref{property:trace_circular} \& \pref{property:trace_sub_cauchy_schwartz} \\
            &=\|\L_{k-1}\|_2\tr(\G_k\R_{k-1}\G_k^\top) \\
            &=\|\L_{k-1}\|_2\tr(\R_{k-1}\G_k^\top\G_k) \\
            &\le \|\L_{k-1}\|_2 \|\R_{k-1}\|_2 \tr(\G_k^\top\G_k)\\
            &=\|\L_{k-1}\|_2 \|\R_{k-1}\|_2 \| \G_k\|_F^2\\
            &\le m\|\L_{k-1}\|_2 \|\R_{k-1}\|_2
        \end{align*}
        where the last line applies equivalence of matrix norms in addition to the 1-Lipschitz loss. Further,
        \begin{align*}
            \|\N_k^2\|_2 &=\frac{\|\G_k\R_{k-1}\G_k^T\|_2}{t_k}\\
            &\le \frac{\|\G_k\R_{k-1}\G_k^T\|_2}{1+\|\widetilde\T_k\|_2}\\
            &\le \|\G_{k} \R_{k\!-\!1}\G_{k}^\top\|_2 \\
            &\le \| \G_k\|_2 \|\R_{k-1}\|_2 \|\G_k^\top\|_2 \\
            &\le \|\R_{k-1}\|_2,
        \end{align*}
        where the last line again uses 1-Lipschitz loss. Putting this together,
        \begin{align*}
            &\tr(\B_{k}) \!-\! \tr(\B_{k\!-\!1})\le  m\|\R_{k-1}\|_2\left(1+\|\L_{k-1}\|_2\|\R_{k-1}\|_2\right)
        \end{align*}
        By telescoping we get
        \begin{align*}
            \tr(\B_{k}) &\!-\! \tr(\B_{0})\le  \sum_{i\in\Ical}m\|\R_{i}\|_2\left(1+\|\L_{i}\|_2\|\R_{i}\|_2\right),
        \end{align*}
        where $\Ical$ is the set of indices for which $\L_k$ was updated. Then $|\Ical|=k'$. Finally, since we know that $\|\L_k\|_2\le\|\L_0\|_2$ and $\|\R_k\|_2\le\|\R_0\|_2$,
        \begin{align*}
            \tr(\B_{k}) &\!-\! \tr(\B_{0})\le k' m\|\R_0\|_2\left(1+\|\L_0\|_2\|\R_0\|_2\right)
        \end{align*}
    \end{proof}

    \subsection{Proof of Theorem \ref{thm:krad_regret}}
    \label{app:proof_krad_regret}
    \begin{proof}
            Plugging in $(\R_{k,1},\L_{k,2})\!=\!(\I_m,\I_n)$ into \eqref{eq:trace1}-\eqref{eq:trace2} and again using the trace H\"{o}lder inequality~(P6) 
            yields\par
            {\centering
                $\begin{aligned}[b]\tr(\B_{k}^{1/4}) &\le \left(m^3\tr(\B_{0}) \!+\! km^4\left(1 \!+\! \|\L_0\|_2\right)\right)^{1/4} \\
                \tr(\C_{k}^{1/4}) &\le \left(n^3\tr(\C_{0}) \!+\! kn^4\left(1 \!+\! \|\R_0\|_2\right)\right)^{1/4}\end{aligned}$.
            \par}
            The term in the regret $\tr(\B_{k}^{1/4})\!\tr(\C_{k}^{1/4})$ is bounded by\par
            {\centering
                $(am^3 \!+\! m^4 bk)^{1
                /4}(cn^3 \!+\! n^4 dk)^{1/4}=O(mn\sqrt{k})$,
            \par}
            where $a,b,c,d$ are constants w.r.t. $m,n,k,k'$ that can be expressed in terms of $\tr(\B_0)$, $\tr(\C_0)$, $\|\L_0\|_2$, and $\|\R_0\|_2$.
        \end{proof}
    
    \section{Implementation details}
    \label{app:implementation}
    We provide a few additional implementation details here.

\subsection{KrADagrad}
\label{app:implementation_kradagrad}
We describe the algorithm for the alternate update scheme, which we dub KrADagrad.
\begin{algorithm}[ht]
       \caption{KrADagrad}
       \label{alg:KrADagrad}
        \begin{algorithmic}
            \STATE {\bfseries Input:} Parameters $\W_0\in\Rbb^{m\times n}$, iterations $K$, update set $\Ical$, step size $\eta$, exponent $\alpha=1/2$
            \STATE {\bfseries Initialize:} $(\L_0, \R_0)=(\I_m, \I_n)$
            \FOR{$k=1, \ldots, K$}
                \STATE Obtain gradient $\g_k$
                \IF{$k\in \Ical$}
                    \STATE Compute $\Delta\L_{k}=\frac{1}{t_k}\L_{k-1}\G_{k}\R_{k-1}\G_{k}^\top \L_{k-1}$
                    \STATE Update $\L_k \leftarrow \L_{k-1}-\Delta\L_{k}$
                \ELSE
                    \STATE Compute $\Delta\R_{k}=\frac{1}{t_k}\R_{k-1}\G^\top_{k}\L_{k-1}\G_{k}\R_{k-1}$
                    \STATE Update $\R_k \leftarrow \R_{k-1}-\Delta\R_{k}$
                \ENDIF
                \STATE Approximate $(\L_k^{\alpha}, \R_k^{\alpha})$ from $(\L_k, \L^{\alpha}_{k\!-\!1}, \R_k, \R^{\alpha}_{k\!-\!1})$
                \STATE Apply preconditioned gradient step
                    \begin{center}
                        $\W_k = \W_{k-1} - \eta\L^{\alpha}_k\G_k\R^{\alpha}_k$
                    \end{center}
            \ENDFOR
        \end{algorithmic}
    \end{algorithm}

\subsection{Delayed updates}
    For KrADagrad$^\star$, we may choose an interval at which to update the preconditioners from the statistics. In this case, we may also use additional memory to batch the updates and reduce the FLOPs needed to compute the statistics to be equal to those needed for Shampoo. To derive the delayed update, consider the modification
    \begin{align*}
        \widetilde\Q_{k+\kappa} &= \Q_{k} + \frac{1}{nt_{k+\kappa}}\sum_{i=1}^{\kappa}\g_{k+i}\g_{k+i}^\top\\
        \Rightarrow \Delta\B_{k+\kappa} &= \frac{1}{nt_{k+\kappa}}\sum_{i=1}^{\kappa}\G_{k+i}\R_{k}\G_{k+i}^\top.
    \end{align*}
    Then, using our extra memory to accumulate $\Delta\B_{k+\kappa}$, we can take
    \[
        t_{k+\kappa}=1 + \tr\left(\L_k \sum_{i=1}^{\kappa}\G_{k+i}\R_{k}\G_{k+i}^\top \right)
    \]
    and 
    \begin{align*}
        \Delta\L_{k+\kappa} &= \frac{1}{nt_{k+\kappa}}\L_{k}\sum_{i=1}^{\kappa}\G_{k+i}\R_{k}\G_{k+i}^\top\L_{k}.
    \end{align*}

\subsection{Matrix Roots}
\label{app:implementation_matrix_roots}

    Matrix roots can be computed using numerical linear algebra methods (e.g. SVD) or iteratively. We covered the use of SVD in the main text, so we discuss iterative methods here. Using parenthesized superscript to denote the iterations, the simplest of these takes the form
    \begin{align*}
        \X^{(k+1)} = f(\X^{(k)}),
    \end{align*}
    where $f$ is a polynomial. More sophisticated methods introduce additional helper matrices. For example, one well-known method for iteratively computing the matrix square root and its inverse without matrix inversion is the Newton-Schulz (NS) method. The NS iteration is given as starting from $\Y^{(0)}=\A, \Z^{(0)}=\I$ and iterating
    \begin{align}
        \begin{cases}
            \X^{(k)} &= \frac{1}{2}(3\I-\Z^{(k)}\Y^{(k)})\\
            \Y^{(k+1)} &= \Y^{(k)}\X^{(k)}\\
            \Z^{(k+1)} &= \X^{(k)}\Z^{(k)}
        \end{cases}
    \end{align}
    until convergence. The values at convergence are $\Y^{(k)}\rightarrow\A^{1/2}, \Z^{(k)}\rightarrow\A^{-1/2}$. The iterations are derived from the so-called ``matrix sign function'' \cite{higham08functions}.
    
    While the NS method converges quadratically, information about the matrix to be rooted $\A$ is encoded not in the iterations themselves or in any invariants, but only in the initialization, making accurate warm starting of the estimate impossible (as far as we are aware). We could also use the numerically stable coupled Newton method for $p$-th roots of $\A$. However, these iterations involve matrix inversions, which are also fairly expensive, even using Newton iterations.
    Instead, we modify the also numerically stable coupled Newton's method for $-p$-th roots and compute $\A^{1/p}=\A(\A^{-1/p})^{p-1}$ \cite{higham08functions}. While it may seem that this requires high precision to compute (as $\A^{-1/p}$ is exactly what Shampoo computes), the difference is that in multiplying by $\A$, the impact of inaccuracies in small eigenvalues remains small.
    Initializing $(\X^{(0)}, \M^{(0)})=(\I, \A)$, we arrive at
    \begin{align}
    \label{eq:coupled_newton_iteration}
        \begin{cases}
            \N^{(k)} &= \frac{1}{p}\left((p+1)\I-\M^{(k)}\right)\\
            \Y^{(k)} &= \N^{(k)}\X^{(k)} \\
            \X^{(k+1)} &= \frac{1}{2}\left(\Y^{(k)} + \Y^{(k)\top}\right)\\
            \M^{(k+1)} &= (\N^{(k+1)})^{p}\M^{(k)},
        \end{cases}
    \end{align}
    where $\X^{(k)}\rightarrow \A^{-1/p}, \M^{(k)}\rightarrow \I$.
    These iterations also appear not to encode information about $\A$; however, for $p=2\phi$, invariant in the iterations is $\M^{(k)}=(\X^{(k)})^{\phi}\A(\X^{(k)})^{\phi}$. For even integers $p$, this gives us exact guidance on how to initialize the iterations with a good guess $\widehat\X\approx\A^{-1/p}$ that leads to the desired fixed point: $(\X^{(0)}, \M^{(0)})=(\widehat\X, \widehat\X^{\phi}\A\widehat\X^{\phi})$. We note that if $\N^{(k)}$ and $\X^{(k)}$ commute, the step of symmetrizing $\Y^{(k)}$ is unnecessary; in practice, due to numerics and especially if we use a warm start, this step ensures that the sequence $\X^{(k)}$ stays symmetric. These iterations also require commutativity of $\A$ and $\X^{(k)}$, and thus non-commutativity may still be a source of divergence. Of course, this initialization procedure now requires additional storage of $\A^{-1/p}$ in between preconditioning computation steps, so we may still opt to use the prescribed initialization depending on the compute-memory requirements of our hardware and software system.

\subsection{Diagonal damping}
\label{app:implementation_damping}
    For numerical reasons, Adagrad and Shampoo add a small scaled identity $\epsilon\I$ prior to matrix inversion (i.e. to avoid singularity). For KrADagrad, while not necessary for numerical stability, we find that adding damping in some cases helps reduce the effect of noise in the gradient estimate on the preconditioner. We hypothesize that it reduces the effect of noise in the direction of eigenvectors with small eigenvalues. Assuming $\epsilon\L_k \prec \I$, we have
    \begin{align*}
    (\epsilon\I + \B_k)^{-1} &= (\epsilon\L_k + \I)^{-1}\L_k\\
    &\succeq\L_k(\I-\epsilon\L_k)\\
    \Rightarrow \left(\L_k(\I-\epsilon\L_k)\right)^{-1} &\succeq \epsilon\I + \B_k.
    \end{align*}
    Then letting $\L'_k = \L_k(\I-\epsilon\L_k)$, this looks like
    \begin{align*}
    \left(\L'_{k} - \frac{1}{t'_k}\L'_{k}\G\R_k\G_k^\top\L'_k\right)^{-1} \succeq
    \epsilon\I + \B_{k} + \frac{1}{t}\G_k\R_{k}\G_k^\top.
    \end{align*}
    This implies that the update should be computed with $\L_k$ while the damped preconditioner should be computed with the modified $\L'_k$. However, as a computational savings, we can incorporate $\L'_k$ into both the update and the preconditioner computations, and suffer negligible error as long as $\epsilon$ is set small enough.

\subsection{Generalization to tensors}
\label{app:implementation_tensors}
    Extending the algorithm to tensors is an exercise in bookkeeping. We first introduce some additional terminology and notation that is standard in the literature for analysis of tensors as well as practical implementation of tensor algebra \cite{lu13multilinear,gupta18shampoo,ren21tensor, paszke17automatic}. For the following, consider two tensors $\A_1$ and $\A_2$ with $D_1$ and $D_2$ dimensions, respectively, 
    \begin{align*}
        \A^{(i)} \in \Rbb^{n_1 \times \ldots \times n_{D_i}}.
    \end{align*}
    \begin{enumerate}
        \item Let $\mathcal{D}_i\subset \{1,...,D_i\}$ be ordered sets $(d_{i,1}, \ldots, d_{i,p})$ with $p=|\mathcal{D}_1|=|\mathcal{D}_2|$. Define the \textit{tensordot} to be such that the element with index $i_1,\ldots,i_{D_1+D_2-2p}$ is
            \begin{align*}
                &[\textrm{TD}(\A^{(1)},\A^{(2)},\mathcal{D}_1, \mathcal{D}_2)]_{i_1,\ldots,i_{D_1+D_2-2p}} \\
                &\quad= \sum_{s_p=1}^{n_{d_{1,p}}}\ldots\sum_{s_1=1}^{n_{d_{1,1}}} \A^{(1)}_{i_1,\ldots,i_{d_{1,1}-1}, s_1, i_{d_{1,1}+1}, \ldots, i_{d_{2,p}-1}, s_p, i_{d_{2,p}+1},\ldots}\A^{(2)}_{i_1,\ldots,i_{d_{2,1}-1}, s_1, i_{d_{2,1}+1}, \ldots, i_{d_{2,p}-1}, s_p, i_{d_{2,p}+1},\ldots},
            \end{align*}
            which is a summation over the indices in the sets $\mathcal{D}_1$ and $\mathcal{D}_2$ of the broadcasted product aligned over the indices specified by $\mathcal{D}_1$ and $\mathcal{D}_2$. Note this is only well defined if $n_{d_{1,k}}=n_{d_{2,k}}$ for $k=1,\ldots,p$. The tensordot is a multidimensional generalization of matrix multiplication and traces. To see this, suppose $\A^{(1)}$, $\A^{(2)}$ are matrices and $\x$, $\y$ are column vectors,
            \begin{align*}
                &\textrm{TD}(\x,\y, [1], [1]) = \x^\top \y\\
                &\textrm{TD}(\A^{(1)},\x, [2], [1]) = \A^{(1)}\x\\
                &\textrm{TD}(\A^{(1)},\A^{(2)}, [2], [1]) = \A^{(1)}\A^{(2)}\\
                &\textrm{TD}(\A^{(1)},\A^{(2)}, [1], [1]) = \A^{(1)\top}\A^{(2)}\\      
                &\textrm{TD}(\A^{(1)},\A^{(2)}, [1,2], [1,2]) = \tr\left(\A^{(1)\top}\A^{(2)}\right) = \A^{(1)} \cdot \A^{(2)}
            \end{align*}
        \item Take $\mathcal{D}_{\backslash d}=[1,\ldots,D]\backslash d$. Let the \textit{contraction} of a tensor be the matrix of size $n_d \times n_d$ resulting from
            \begin{align*}
                \textrm{contr}_d(\A) = \textrm{TD}(\A,\A,\mathcal{D}_{\backslash d},\mathcal{D}_{\backslash d})
            \end{align*}
        \item For $1\le d\le D$ and matrix $\L_d\in\Rbb^{n_d \times m}$, let the $d$-th tensor-matrix product be
            \begin{align*}
                \A \times_d \L_d = \textrm{TD}(\L_d, \A, [2], [d]).
            \end{align*}
    \end{enumerate}

    \begin{algorithm}[tb]
       \caption{KrADagrad$^\star$ for tensors}
       \label{alg:KrADagrad_star_tensor}
        \begin{algorithmic}
            \STATE {\bfseries Input:} Parameter tensor $\W_0\in\Rbb^{n_1\times \cdots \times n_D}$ of order $D$, iterations $K$, step size $\eta$, damping $\epsilon$, exponent $\alpha = 1/2$
            \STATE {\bfseries Initialize:} $\L_{0,d}=\I_{n_d}$
            \FOR{$k=1, \ldots, K$}
                \STATE Obtain gradient $\G_k$
                \FOR{$d=1, \ldots, D$}                    
                    \STATE Compute
                        \begin{center}
                            $\Delta\L_{k,d}=\frac{1}{t_{k,d}}\L_{k\!-\!1,d}\textrm{contr}_d(\G_{k}) \L_{k\!-\!1,d}$
                        \end{center}
                    \STATE Update $\L_{k,d} \leftarrow \L_{k-1,d}-\Delta\L_{k,d}$
                    \STATE Compute $\L^{\alpha/D}_{k,d}$ from $\L_{k,d}$
                \ENDFOR
                \STATE Apply preconditioned gradient step
                    \begin{center}
                        $\W_k = \W_{k-1} - \eta\G_k \times_1\L^{\alpha/D}_{k,1} \cdots \times_D \L^{\alpha/D}_{k,D}$
                    \end{center}
            \ENDFOR
        \end{algorithmic}
    \end{algorithm}

    \section{Properties and extended derivations}
    \label{app:derivations}
    Here we list the main linear algebra properties utilized in derivations and proofs. We then extend the derivations as a mini-tutorial in matrix manipulation for maximum clarity.

\subsection{Linear algebra properties}
    Matrix shapes are assumed to be conforming when left unspecified. 
    \begin{property}
    \label{property:trace_permute}
        $\A\cdot\B = \B \cdot \A$.
    \end{property}
    
    \begin{property}
    \label{property:trace_circular}
        \pref{property:trace_permute} implies $\tr(\A\B\C) = \tr(\B\C\A)$.
    \end{property}
    
    \begin{property}
    \label{property:trace_PSD_pos_inner_prod}
        If $\A,\B\succ 0$, then $\A\cdot\B > 0$.
    \end{property}
    
    \begin{property}
    \label{property:trace_ge_2norm}
        If $\A\succeq0$, then $\|\A\|_2 \le \tr(\A)$.
    \end{property}
    
    \begin{property}
    \label{property:trace_sub_cauchy_schwartz}
        If $\A,\B\succeq 0$, then $\A\cdot\B \le \|\A\|_2\tr(\B)$.
    \end{property}
    
    \begin{property}
    \label{property:trace_holder}
        If $\A,\B\succ 0$, then for $p,q$ such that $1/p+1/q=1$, $\A^{1/p}\cdot\B^{1/q} \le (\tr(\A))^{1/p}(\tr(\B))^{1/q}$.
    \end{property}
    
    \begin{property}
    \label{property:trace_power}
        \pref{property:trace_ge_2norm} and \pref{property:trace_sub_cauchy_schwartz} together imply that for $\A\succeq 0$, $\tr(\A^i) \le \|\A\|_2^{i-1} \tr(\A) \le \tr(\A)^i \quad \textrm{for } i \ge 1$.
    \end{property}
    
    \begin{property}
    \label{property:frob_spectral}
        \pref{property:trace_sub_cauchy_schwartz} implies that for $\A,\B \succeq 0$, $\|\A\B\|_F \le \|\A\|_2 \|\B\|_F$.
    \end{property}

    \begin{property}\label{property:domination_operator_monotone}
        If $\A\succeq \B\succ 0$, then for $0\le p\le 1$, $\A^p\succeq\B^p$.
    \end{property}
    
    \begin{property}\label{property:domination_quadratic}
        If $\A,\B,\C\succeq 0$, then $\A\succeq\B \Rightarrow \C\A\C^\top\succeq\C\B\C^\top$.
    \end{property}

    \begin{property}\label{property:domination_diagonal}
        If diagonal matrices $\D\succeq\C\succ 0$, then $\C^{-1}\succeq\D^{-1}$.
    \end{property}
    
    \begin{property}\label{property:domination_sum_inverse}
        \pref{property:domination_diagonal} implies that if $\B\succeq 0, \C \succ 0$, then $\C^{-1}\succeq(\C+\B)^{-1}$.
    \end{property}
    
    \begin{property}
    \label{property:trace_kron}
        $\tr(\A\otimes\B) = \tr(\A)\tr(\B)$.
    \end{property}
    
    \begin{property}
    \label{property:kron_prod}
        $(\A\otimes\B)(\C\otimes\D) = (\A\C)\otimes(\B\D)$.
    \end{property}

    \begin{property}\label{property:kron_prod_distributive}
        $\A\otimes (\B+\C)=\A\otimes \B+\A\otimes \C$.
    \end{property}
    
    \begin{property}\label{property:kron_power} If $\A,\B \succeq0$, then $(\A\otimes\B)^{r} \!=\! \A^{r}\otimes\B^{r} \; \forall r\in\Rbb$.
    \end{property}
    
    \begin{property}\label{property:kron_transpose} $(\A\otimes\B)^{\top}=\A^{\top}\otimes\B^{\top}$.
    \end{property}
    
    \begin{property}\label{property:kron_vec} $(\C^\top\otimes\A)\vec(\B)=\vec(\A\B\C)$.
    \end{property}

\subsection{Additional and extended derivations}
    \subsubsection{Proof of~\pref{property:domination_quadratic}}
        For arbitrary $\x$, let $\y=\C^\top \x$. We have,
        \begin{align*}
            \x^\top \C \A \C^\top \x &= \y^\top \A \y \\
            &\ge \y^\top \B \y & \because \A\succeq\B \\
            &=\x^\top \C \B \C^\top \x
        \end{align*}

    \subsubsection{Proof of~\pref{property:domination_sum_inverse}}
    Let the SVD of $\C^{-1/2}\B\C^{-1/2} = \V\boldsymbol{\Sigma}\V^\top$. Then,
    \begin{align*}
        \C^{-1} &= \C^{-1/2}\I\C^{-1/2} \\
        &= \C^{-1/2}\V\I\V^\top\C^{-1/2}  & \because \V^{-1} = \V^\top\\
        &\succeq \C^{-1/2}\V(\I+\boldsymbol{\Sigma})^{-1}\V^\top\C^{-1/2} & \because \pref{property:domination_quadratic}\&\pref{property:domination_diagonal}\\
        &=\C^{-1/2}(\V\V^\top+\V\boldsymbol{\Sigma}\V^\top)^{-1}\C^{-1/2}\\
        &= \C^{-1/2}(\I+\C^{-1/2}\B\C^{-1/2})^{-1}\C^{-1/2}\\
        &= (\C+\B)^{-1}
    \end{align*}
    \subsubsection{Extended derivation of KrADagrad update}
        We can use Woodbury directly on the inverse before applying KrAD,
        \begin{align*}
            \widetilde\Q_{k}^{-1} &= (\Q_{k-1} + \frac{1}{nt_k}\g_k \g_k^\top)^{-1} \\
            &= \Q_{k-1}^{-1} - \frac{1}{nt_k (1+\g_k^\top\Q_{k-1}^{-1}\g_k)}\Q_{k-1}^{-1} \g_k \g_k^\top\Q_{k-1}^{-1} \\
            \Rightarrow \Q_{k-1}^{-1} - \widetilde\Q_{k}^{-1} &= \frac{1}{nt_k (1+\tr(\G_k^\top\L_{k-1}\G_k\R_{k}))}\Q_{k-1}^{-1} \g_k \g_k^\top\Q_{k-1}^{-1}. 
        \end{align*}
        Now using KrAD on the RHS above,
        \begin{align*}
             \frac{1}{n}\Q_{k-1}^{-1} \g_k \g_k^\top\Q_{k-1}^{-1} &= \frac{1}{n}\vec(\L_{k-1} \G_k\R_{k}) \vec(\L_{k-1} \G_k\R_{k})^\top\\
             &\preceq \R_k \otimes (\L_{k-1}\G_k\R_{k}\R_{k}^{-1}\R_{k}\G_k^\top\L_{k-1})\\
             &= \R_k \otimes (\L_{k-1}\G_k\R_{k}\G_k^\top\L_{k-1})
        \end{align*}
        so that
        \begin{align*}
             \widetilde\Q_{k}^{-1} \succeq \R_{k}\otimes\L_{k-1} - \frac{1}{t_k (1+\tr(\G_k^\top\L_{k-1}\G_k\R_{k}))}\R_k \otimes (\L_{k-1}\G_k\R_{k}\G_k^\top\L_{k-1}).
        \end{align*}
        Then letting
        \[
            \Delta\L_{k}=\frac{1}{1+\tr(\G_k^\top\L_{k-1}\G_k\R_{k})}\L_{k-1}\G_k\R_{k}\G_k^\top\L_{k-1},
        \]
        take $t_k=1$ and $\L_k=\L_{k-1}-\Delta\L_{k-1}$, and we have
        \begin{align*}
             \Q_k^{-1} = \R_k\otimes\L_{k} \preceq \widetilde\Q_k^{-1}\preceq \left(\frac{1}{n}\Gcal_k\Gcal_k^\top\right)^{-1}.
        \end{align*}
    \subsubsection{Alternative derivation of KrADagrad update}
        We can also start with
        \begin{align*}
            &\G_k^T \L_{k-1}\G_k \succeq 0\\
            \Rightarrow&\G_k^T \L_{k-1}\G_k+\C_k \succeq \C_k\\
            \Rightarrow &\C_k^{-1}\succeq (\G_k^T \L_{k-1}\G_k+\C_k)^{-1}, &\because \pref{property:domination_sum_inverse}
        \end{align*}
        which implies that
        \begin{align*}
             &\! \frac{1}{t_k}\L_{k\!-\!1}\G_k(\C_{k})^{-1}\G_k^{\top}\L_{k\!-\!1} \succeq \! \frac{1}{t_k}\L_{k\!-\!1}\G_k(\C_{k}\!+\!\frac{1}{t_k}\G_k^{\top}\L_{k\!-\!1}\G_k)^{-1}\G_k^{\top}\L_{k\!-\!1} & \because \pref{property:domination_quadratic}
        \end{align*}
        and thus
        \begin{align*}
             - & \frac{1}{t_k}\L_{k\!-\!1}\G_k(\C_{k}\!+\!\frac{1}{t_k}\G_k^{\top}\L_{k\!-\!1}\G_k)^{-1}\G_k^{\top}\L_{k\!-\!1} \succeq - \frac{1}{t_k}\L_{k\!-\!1}\G_k(\C_{k})^{-1}\G_k^{\top}\L_{k\!-\!1}
        \end{align*}

    \section{Memory and compute cost analysis}
    \label{app:costs}
    \subsection{Kradagrad: statistic update}
    For Kradagrad, we can actually cache some of the computations from the statistic update for later use in applying the preconditioner. Computing $\Delta\L_k$ has a computational cost of $2N(m+n)+4m^2+4m^3\le 4m(2n^2+n)$ from the following:
    \begin{enumerate}
        \item $\G_k\R_{k-1}^{1/2}\rightarrow 2mn^2 = 2Nn$
        \item $\G_k\R_{k-1}^{1/2}(\G_k\R_{k-1}^{1/2})^{\top}\rightarrow 2m^2 n = 2Nm$
        \item $\L_{k-1}(\G_k\R_{k-1}\G_k^\top)\rightarrow 2m^3$
        \item $t_k = 1 + \|\L_{k-1}\G_k\R_{k-1}\G_k^\top\|_F\rightarrow 2m^2$
        \item $\Delta\L_k = (\L_{k-1}\G_k\R_{k-1}\G_k^\top)\L_{k-1}\rightarrow 2m^3$
        \item $\frac{1}{t_k}\Delta\L_k \rightarrow m^2$
        \item $\L_{k-1} - \frac{1}{t_k}\Delta\L_k \rightarrow m^2$
    \end{enumerate}
    Similarly, computing $\Delta\R_k$ has a computational cost upper bounded by $2N(m+n)+4n^2+4n^3 \le 4n(2n^2+n)$. Thus, using a simple alternating scheme, the per update computation cost is bounded by the average, $2(m+n)(2n^2+n)\rightarrow O(n^3)$.
    
    The memory cost is $m^2+n^2\rightarrow O(n^2)$ beyond gradient descent. We may cache $\G_k\R_{k-1}^{1/2}$ replacing $\G_k$ in memory, which incurs no additional storage cost.
\subsection{Kradagrad: matrix square root}
    The cost comes from NS iterations or SVD. For each NS iteration, the cost is $O(n^3)$, and the SVD also has cost $O(n^3)$.

\subsection{Kradagrad: applying preconditioner}
    In iterations where we update $\L_k$, using the cached $\G_k\R_{k-1}^{1/2}$, we merely need to compute
    \[
        \L_{k-1}^{1/2}(\G_k\R_{k-1}^{1/2})\rightarrow 2m^2 n = 2Nm.
    \]
    For iterations in which we update $\R_k$, we similarly have a cost of $2Nn$. Again using a simple alternating scheme, the average cost is then $(m+n)N\rightarrow O(n^3)$.

\subsection{Kradagrad$^{\star}$: statistics updates}
    For Kradagrad$^{\star}$, computing $\Delta\L_k$ has a computational cost of $(2Nm+2m^2+4m^3)= 2m^2(n+2m+1)$ from the following:
    \begin{enumerate}
        \item $\G_k\G_k^\top\rightarrow 2m^2 n = 2Nm$
        \item $\L_{k-1}(\G_k\G_k^\top)\rightarrow 2m^3$
        \item $\|\L_{k-1}\G_k\G_k^\top\|_F\rightarrow 2m^2$
        \item $(\L_{k-1}\G_k\G_k^\top)\L_{k-1}\rightarrow 2m^3$
    \end{enumerate}
    Similarly, computing $\Delta\R_k$ has a computational cost of $(4Nn+2n^2+2n^3)= 2n^2(n+2m+1)$ from the following:
    \begin{enumerate}
        \item $\R_k\G_k^\top\rightarrow 2mn^2 = 2Nn$
        \item $(\R_{k-1}\G_k^\top)\G_k\rightarrow 2mn^2 = 2Nn$
        \item $\|\R_{k-1}\G_k^\top\G_k\|_F\rightarrow 2n^2$
        \item $(\R_{k-1}\G_k^\top\G_k)\R_{k-1}\rightarrow 2n^3$
    \end{enumerate}
    Thus the total per update compute cost is $2(m^2+n^2)(n+2m+1)\rightarrow O(n^3)$.
    
    Again the memory cost is straightforward at $m^2+n^2\rightarrow O(n^2)$.
\subsection{KrADagrad$^{\star}$: matrix $p$-th roots}
    The cost comes from the coupled Newton iterations. The cost of a couple Newton iteration is $O(n^3)$, and the SVD also has cost $O(n^3)$.
    
\subsection{Kradagrad$^{\star}$: applying preconditioners}
    The cost is
    \begin{enumerate}
        \item $\L_k^{1/2p}\G_k\rightarrow 2m^2 n = 2Nm$
        \item $(\L_k^{1/2p}\G_k)\R_k^{1/2p}\rightarrow 2mn^2 = 2Nn$
    \end{enumerate}
    for a total of $2(m+n)N\rightarrow O(n^3)$.

    \section{Other Kronecker approximations}
    \label{app:other_kron}
    Here we describe a traditional Kronecker approximation that may point in a direction for further research.
Given matrix $\Ucal\in\Rbb^{mn\times r}$, we wish to minimize the Frobenius norm
\begin{align}
\label{eq:LR_optimization}
    (\L,\R) = \underset{\L',\R'}{\argmin} \left\|\R'\otimes\L' - \Ucal\Ucal^\top\right\|_F^2.
\end{align}
Note that there is an equivalence class for solutions, i.e. if $(\B, \C)$ is a solution, then so is $(\alpha\B, \frac{1}{\alpha}\C)$ for $\alpha\ne 0$. One solution is related to taking the singular vectors corresponding to the largest singular value of a permuted and reshaped matrix~\cite{vanloan93approximation}. Alternatively, for large matrices, we can apply alternating minimization to obtain iterative updates for $\L^{(t+1)}$ at iteration $t+1$ given $\R^{(t)}$ at iteration $t$ and vice versa,
\begin{align}
    \L^{(t+1)} &= \underset{\L}{\argmin} \left\|\R^{(t)}\otimes\L - \Ucal\Ucal^\top \right\|_F^2\\
    \R^{(t+1)} &= \underset{\R}{\argmin} \left\|\R\otimes\L^{(t+1)} - \Ucal\Ucal^\top \right\|_F^2.
\end{align}
Under mild conditions, this process converges to a local minimum~\cite{tseng01convergence}. We can choose to perform half an iteration, iterate until convergence, or anywhere in between.

Letting $\Ucal_i$ be the $i$-th column of $\Ucal$, and $\U_i=\vec^{-1}_{m,n}(\Ucal_i)$, note that
\begin{align*}
    &\tr\left( \Ucal^\top(\R^{(t)}\otimes\L) \Ucal\right)\\
    &\quad =\sum_{i=1}^{r}\Ucal_i^{\top}(\R^{(t)}\otimes\L) \Ucal_i\\
    &\quad =\sum_{i=1}^{r}\Ucal_i^{\top}\vec(\L \U_i \R^{(t)}) \\
    &\quad =\sum_{i=1}^{r}\tr(\L \U_i \R^{(t)} \U_i^{\top} )
\end{align*}
Then, we can rewrite and take the derivative of the objective function (dropping irrelevant constants) and set it to zero to solve for the partial optimum,
\begin{align*}
    J(\B) &=-2 \tr\left(\Ucal\Ucal^\top (\R^{(t)}\otimes\L)\right) + \tr((\R^{(t)\top}\R^{(t)})\otimes(\L^\top\L)) \\
        &= -2 \sum_{i=1}^{r}\tr(\L \U_i \R^{(t)} \U_i^{\top} ) + \tr(\L^\top\L)\tr(\R^{(t)\top}\R^{(t)})\\
    \Rightarrow \frac{\partial J}{\partial \B} &= 2\left\|\R^{(t)}\right\|_F^2\L -2\sum_{i=1}^{r} \U_i \R^{(t)} \U_i^\top = 0 \\
    \Rightarrow \L^{(t+1)} &= \frac{1}{\left\|\R^{(t)}\right\|_F^2}\sum_{i=1}^{r} \U_i \R^{(t)} \U_i^{\top}.
\end{align*}
Similarly,
\begin{align*}
    \R^{(t+1)} &= \frac{1}{\left\|\L^{(t+1)}\right\|_F^2}\sum_{i=1}^{r} \U_i^\top \L^{(t+1)} \U_i.
\end{align*}
If we use traditional approximation to form the basis of a preconditioner, the resulting updates, while simple in form, would not satisfy the analogous properties in Equations~\eqref{eq:shampoo_dominate_grad}-\eqref{eq:shampoo_dominate_prev} that would guarantee optimal regret. Thus, the traditional Kronecker approximation requires another set of properties that lead to optimal regret.

    \section{Experiment Details}
    \label{app:addl_exp}
    \label{app:experiments}

\subsection{Synthetic}

We form $\A=\U\boldsymbol{\Lambda} \U^\top$, where $\U^\top$ is a random orthonormal matrix (generated by taking the $Q$ from a $QR$ decomposition of a square matrix with i.i.d. $\mathcal{N}(0, 1)$ entries), and $\boldsymbol{\Lambda}$ is a diagonal matrix with entries $\lambda_i=10^{(\sqrt{10}-\sqrt{10}(i-1)/n)^2}$ for $i=1,\ldots,n$ (where the square/square root nonlinearities in the exponent make the spread in the top eigenvalues less disparate while not decreasing the overall condition number). We repeat a similar process for $\B$. We can see that the optimal is $\X=\0$ and that the Hessian of the loss function w.r.t. $\x=\vec(\X)$ is $2\B\otimes \A$ with a condition number of $\kappa(\B\otimes\A)=\kappa(\A)\kappa(\B)=10^{20}$, which is exceptionally badly conditioned. We found that for the preconditioned methods an interval of 3 steps in between updating the preconditioner performed best.

\begin{figure}[t]
\begin{center}
\includegraphics[width=0.5\textwidth]{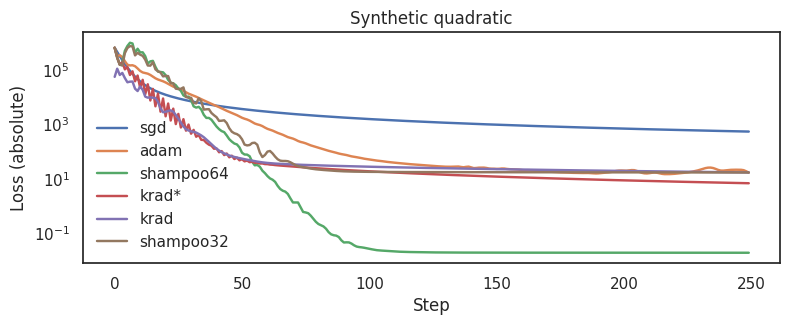}
\end{center}
\caption{Loss on synthetic quadratic (we do not observe pathologically bad behavior)}
\label{fig:quad_abs}
\end{figure}

\begin{center}
\begin{table}[!h]
\caption{Selected learning rate (epsilon) for synthetic quadratic.\label{table:synthetic}}
\centering
\begin{tabular}{ |c|c|c|c|c| } 
 \hline
 \multicolumn{5}{|c|}{Synthetic} \\
 \hline
 SGD & Adam & Shampoo & KrAdagrad$^\star$ & KrAdagrad \\
 \hline
 1e-3, - & 1e0, 1e-8 & 5e-3, 1e-8 & 5e-3, 1e-8 & 1e-2, 1e-8 \\
 \hline
\end{tabular}
\end{table}
\end{center}
The learning rate was swept over the range [1e-3, 2e-3, 5e-3, 1e-2, 2e-2, 5e-2, 1e-1, 2e-1, 5e-1, 1e0] for all methods, while epsilon was swept log uniformly with 6 steps over the range [1e-3, 1e-8] for all methods.

\subsection{Vision Experiments}

We train a ResNet32 model on CIFAR-10 and a ResNet56 model on CIFAR-100, following an implementation by \citep{idelbayev18aproper}, modified to be without (BN).

\begin{figure}[t]
\includegraphics[width=0.5\textwidth]{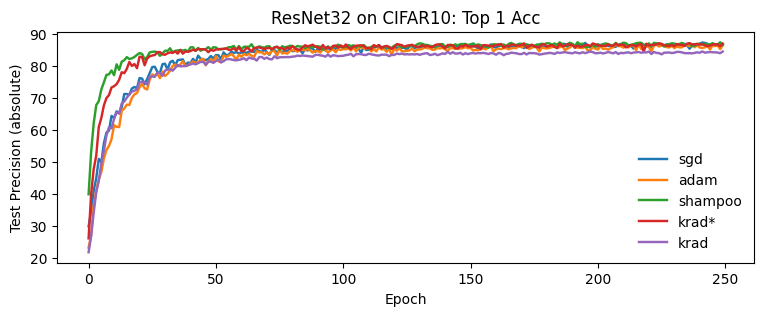}\includegraphics[width=0.5\textwidth]{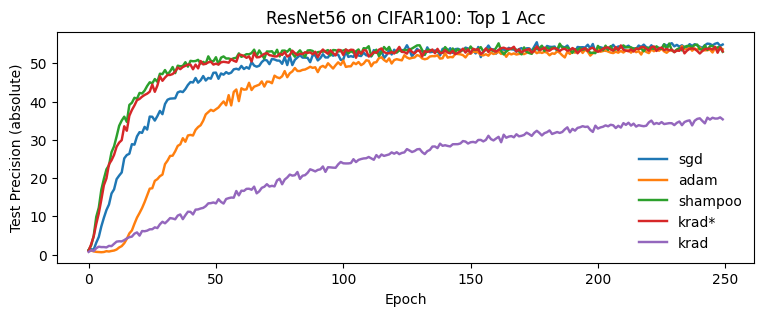}
\caption{Left: Top 1 Accuracy of ResNet-32 (without BN) on CIFAR-10; Right: Top 1 Accuracy of ResNet-56 (without BN) on CIFAR-100. We do not observe pathologically bad behavior.}
\label{fig:resnet_abs}
\end{figure}

Figure \ref{fig:resnet_abs} shows the absolute performance of these two models on the respective datasets they were trained on. The learning rate was swept over the range [1e-4, 2.5e-4, 1e-3, 2.5e-3, 1e-2, 2.5e-2, 1e-1] for preconditioned methods and [1e-3, 2e-3, 5e-3, 1e-2, 2e-2, 5e-2, 1e-1] for SGD and Adam, while epsilon was swept log uniformly with 6 steps over the range [1e-1, 1e-6] for all methods, and the batch size was set to $128$ for all methods.
\begin{center}
\begin{table}[!h]
\caption{Selected learning rate (epsilon) for ResNet-32/56 on CIFAR-10/100 comparisons.\label{table:resnet_hps}}
\centering
\begin{tabular}{ |c|c|c|c|c| } 
 \hline
 \multicolumn{5}{|c|}{ResNet-32 on CIFAR-10} \\
 \hline
 SGD & Adam & Shampoo & KrAdagrad$^\star$ & KrAdagrad \\
 \hline
 5e-3, - & 5e-3, 1e-1 & 1e-2, 1e-4 & 1e-2, 1e-4 & 2.5e-2, 1e-4 \\
 \hline
 \multicolumn{5}{|c|}{ResNet-56 on CIFAR-100} \\
 \hline
 SGD & Adam & Shampoo & KrAdagrad$^\star$ & KrAdagrad \\
 \hline
 1e-2, - & 5e-3, 1e-3 & 1e-2, 1e-4 & 1e-2, 1e-4 & 1e-2, 1e-4 \\
 \hline
\end{tabular}
\end{table}
\end{center}




\subsection{Autoencoder Experiments}

We closely follow \citep{goldfarb20practical} Appendix D.2 Tables 3 and 4 in defining the datasets and hyperparameters for the autoencoder benchmarks. We first perform a hyperparameter search for each optimizer/dataset pair over the learning rate and epsilon parameters. For all optimizers and datasets we sample the learning rate log uniformly between $10^{-5}$ and $2$, and sample epsilon log uniformly on the intervals $[10^{-10}, 10^{-1}]$ and $[10^{-4},1]$ for Adam and 2nd order optimizers, respectively.  For each dataset we run 30 trials for SGD and $\sim 90$ trials for the other optimizers. We use a batch size of 100, momentum of 0.9 for SGD, and Shampoo and KrAdagrad$^\star$ use block sizes of 100, 250, and 500 for the CURVES, MNIST, and FACES datasets, respectively. The best performing hyperparameters are provided in Table \ref{table:ae_hps}. After finding the best HPs from a single seed, we retrain with the selected HPs on five new, unique seeds shared across optimizers. See the raw learning curves in Figure \ref{fig:ae_seeds} and averages in Figure \ref{fig:ae}.

\begin{figure*}
    \centering
    \includegraphics[width=0.33\textwidth]{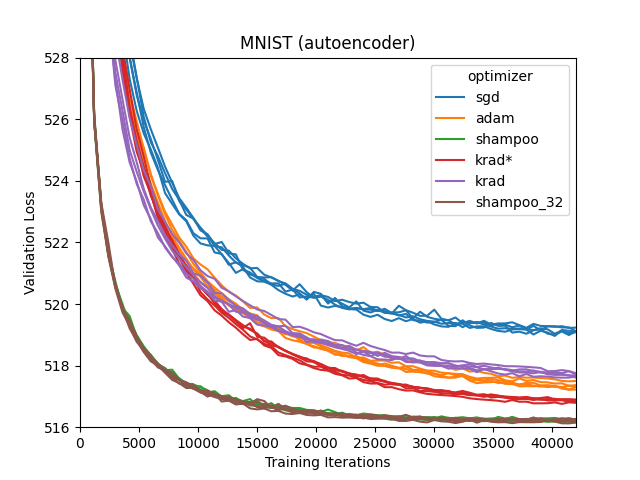}
    \includegraphics[width=0.33\textwidth]{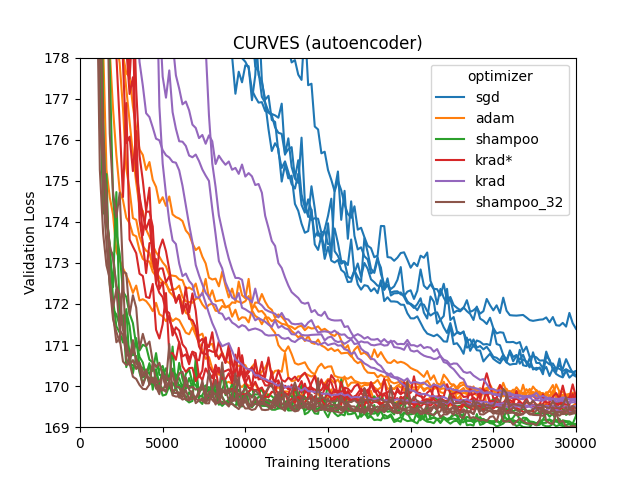}
    \includegraphics[width=0.33\textwidth]{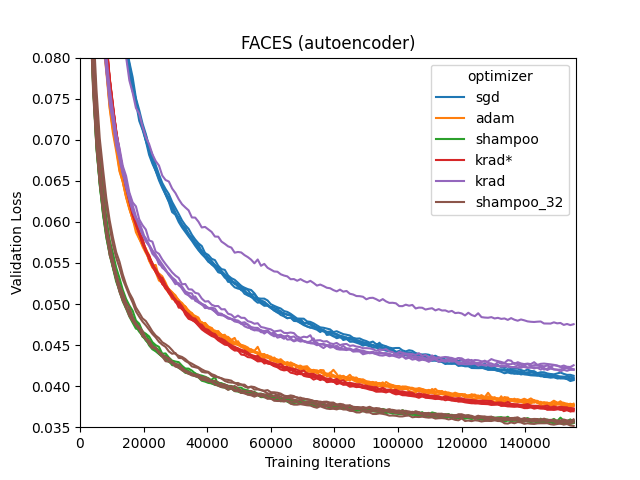}
    \caption{Reconstruction validation error curves for the best HPs for fully connected auto encoder a) mean cross entropy on MNIST b) mean cross entropy on CURVES c) mean squared error on FACES.  All 5 seeds are shown per optimizer.}
    \label{fig:ae_seeds}
\end{figure*}

\begin{center}
\begin{table}[!h]
\caption{Selected learning rate (epsilon) for autoencoder comparisons.\label{table:ae_hps}}
\centering
\begin{tabular}{ |c|c|c|c|c| } 
 \hline
 \multicolumn{5}{|c|}{MNIST} \\
 \hline
 SGD & Adam & Shampoo & KrAdagrad* & KrAdagrad \\ 
 \hline
 5.2e-3 , - & 3e-4 , 1e-4 & 1e-3 , 1e-2 & 1.4e-3 , 3e-4 & 6.3e-3 , 8.5e-4 \\ 
 \hline
 \multicolumn{5}{|c|}{CURVES} \\
 \hline
 SGD & Adam & Shampoo & KrAdagrad* & KrAdagrad \\ 
 \hline
 1e-3 , - & 1e-3 , 1e-3 & 1.9e-3 , 5.8e-4 & 3e-3 , 3.2e-2 & 1.3e-2 , 5e-3 \\ 
 \hline
 \multicolumn{5}{|c|}{FACES} \\
 \hline
 SGD & Adam & Shampoo & KrAdagrad* & KrAdagrad \\ 
 \hline
 1e-1 , - & 3e-3 , 1e-3 & 3.1e-1 , 3.6e-3 & 1.9e-1 , 2e-4 & 3.2e-1 , 1e-4 \\ 
 \hline
\end{tabular}
\end{table}
\end{center}

\subsection{Recommender Experiment}

For the MovieLens-20M experiment we start from the existing H+Vamp Gated baseline\footnote{https://github.com/psywaves/EVCF} and modify it to compare the optimizers of interest.  Using the default hyperparameters otherwise, we find that learning rates greater than or equal to ~4e-4 cause the training loss to diverge for most of the optimizers.  As a result we use the following learning rate grid search, and report the best performing rates in \ref{table:recsys_lrs}:

\begin{itemize}
  \item  SGD, Adam (norm grad)[
  \begin{itemize}
   \item learning rates 1.25e-5, 2.5e-5, 5e-5, 1e-4, 2e-4, 4e-4]
   \item epsilon 1e-6 (Adam only)
  \end{itemize}
  \item Adam, Shampoo, KrAdagrad*, KrAdagrad
  \begin{itemize}
   \item learning rates [6.25e-6, 1.25e-5, 2.5e-5, 5e-5, 1e-4 2e-4]
   \item epsilon 1e-4 (1e-3 for KrAdagrad)
  \end{itemize}
\end{itemize}

\begin{center}
\begin{table}[h!]
\caption{Selected learning rate (epsilon) for MovieLens 20M experiment.\label{table:recsys_lrs}}
\centering
\begin{tabular}{ |c|c|c|c|c|c| } 
 \hline
 SGD & Adam & Shampoo & KrAdagrad* & KrAdagrad & Adam (norm grad) \\
 \hline
 2e-4 & 2e-4 & 2e-4 & 2e-4 & 2-e4 & 4e-4 \\
 \hline
\end{tabular}
\end{table}
\end{center}

For this experiment, we instead tune against the training loss (that which is optimized directly) after discovering that the equivalent validation loss function weights the loss components differently, as described in Figure \ref{fig:hvamp}.

\subsection{Continual Learning Experiments}
In the Permuted MNIST benchmark \citep{permuted_mnist}, the learner trains for 2000 steps, seeing a new pixel permutation of the images every 100 steps, resulting in 20 sequential non-overlapping experiences.  Every 100 steps the Top 1 Accuracy is evaluated against all 20 experiences, including those not yet seen, yielding an aggregate measure of model accuracy in both current and prior domains. For the Split CIFAR-100 benchmark, the learner trains on exactly 5 unseen classes at a time every interval of 2500 training steps, evaluating Top 1 Accuracy against all 20 experiences after each interval, until reaching 50000 training steps. Starting with the GEM and LaMAML baselines \footnote{https://github.com/ContinualAI/continual-learning-baselines} from~\cite{lomonaco2021avalanche}, we train with 3 unique seeds shared across optimizers, with learning curves of Top 1 Accuracy  averaged across seeds for the best HPs in Figure \ref{fig:pmnist_lamaml_actual}. We share the same curves relative to SGD in Figure \ref{fig:pmnist-lamaml}a, and relative to Adam in Figure \ref{fig:pmnist-lamaml}b. We perform the following grid searches to select a best learning rate per optimizer (see Table \ref{table:gem_lrs} for the selected learning rates):

\begin{itemize}
  \item SGD, Adam
  \begin{itemize}
   \item learning rates [1e-1, 5e-2, 2e-2, 1e-2, 5e-3, 2e-3, 1e-3, 5e-4, 2e-4, 1e-4]\footnote{the three smallest rates were included only for tuning Adam for LaMAML}
   \item epsilon 1e-6 (Adam only)
  \end{itemize}
  \item Shampoo, KrAdagrad, KrAdagrad*
  \begin{itemize}
   \item learning rates [1, 2.5e-1, 1e-1, 2.5e-2, 1e-2]
   \item epsilon 1e-4
  \end{itemize}
\end{itemize}

\begin{center}
\begin{table}[!h]
\caption{Selected learning rates for continual learning comparisons.\label{table:gem_lrs}}
\centering
\begin{tabular}{ |c|c|c|c|c| }
\hline
 \multicolumn{5}{|c|}{GEM on Permuted MNIST} \\
 \hline
 SGD & Adam & Shampoo & KrAdagrad* & KrAdagrad \\ 
 \hline
 5e-3 & 1e-3 & 1e-2 & 1e-2 & 1e-2 \\ 
 \hline
 \multicolumn{5}{|c|}{LaMAML on Split CIFAR100} \\
 \hline
 SGD & Adam & Shampoo & KrAdagrad* & KrAdagrad \\ 
 \hline
 2e-2 & 1e-4 & 1e-2 & 1e-2 & 1e-2 \\
 \hline
\end{tabular}
\end{table}
\end{center}
Additionally, for both experiments we use a batch size of 10, a block size of 128 for the 2nd order optimizers, and beta=0.9 for SGD with momentum.
\begin{figure}[t]
\includegraphics[width=0.5\textwidth]{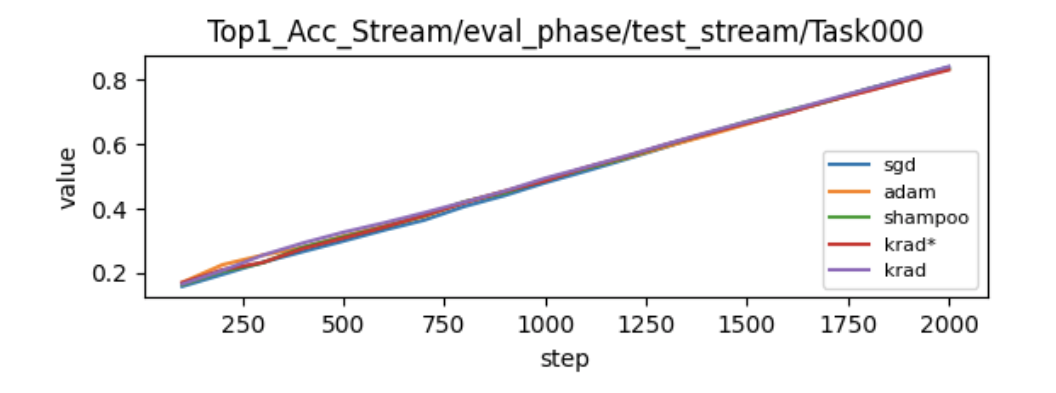}\includegraphics[width=0.5\textwidth]{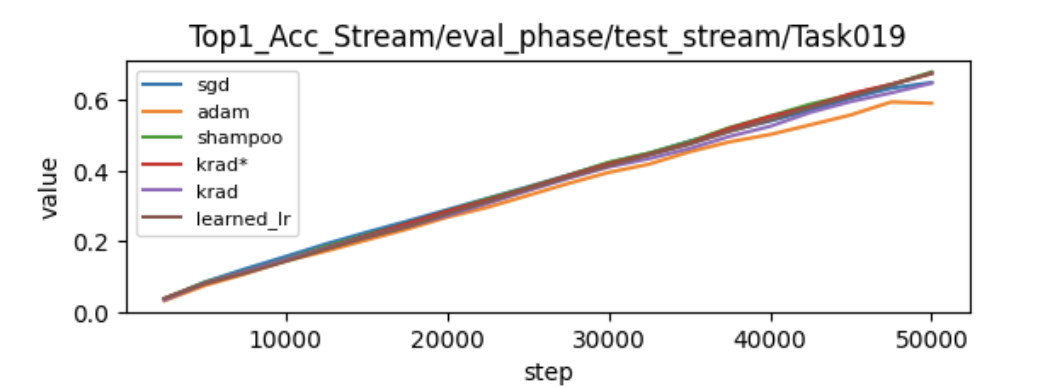}
\caption{Left: Top 1 Accuracy of GEM on Permuted MNIST; Right: Top 1 Accuracy of LaMAML on Split CIFAR-100. We do not observe pathologically bad behavior. The learning curves are extremely similar; hence the relative curves presented in the main body.}
\label{fig:pmnist_lamaml_actual}
\end{figure}

    \stopcontents[appendices]

\end{document}